\documentclass{article}

\usepackage[nonatbib, final, main]{styles/neurips_2026}

\usepackage{microtype}
\usepackage[pagebackref=true,colorlinks,citecolor=brown]{hyperref}
\usepackage{centernot}

\usepackage{amssymb}
\usepackage{afterpage}  
\usepackage{wrapfig}  

\usepackage{graphicx}
\usepackage{subcaption}
\usepackage{booktabs}
\usepackage{multirow}
\usepackage[table]{xcolor}
\usepackage{caption}
\usepackage{amsmath}
\usepackage{amsthm}
\usepackage{cleveref}
\usepackage{amsfonts}

\newtheorem{theorem}{Theorem}

\usepackage{longtable}
\usepackage[most]{tcolorbox}

\usepackage{algorithm}
\usepackage{algorithmic}

\newcommand{\good}[1]{{\boldmath\textbf{#1}}}

\NewDocumentCommand{\alexi}
{ mO{} }{\textcolor{red}{\textsuperscript{\textit{Alexi}}\textsf{\textbf{\small[#1]}}}}

\definecolor{RoyalBlue}{cmyk}{1,0.50,0,0}  
\newtcolorbox{notebox}[1]{
  enhanced,
  breakable,
  colback=RoyalBlue!5!white,
  colframe=RoyalBlue!70!black,
  colbacktitle=RoyalBlue!70!black,
  coltitle=white,
  fonttitle=\bfseries\small,
  fontupper=\small,
  boxrule=0.6pt,
  arc=1.2mm,
  left=3mm, right=3mm, top=1mm, bottom=1mm, toptitle=0.4mm, bottomtitle=0.4mm,
  before skip=4pt, after skip=4pt,
  title={#1},
}

\newtcolorbox{promptbox}[1]{
  enhanced,
  breakable,
  title=#1,
  colback=gray!5,
  colframe=gray!55!black,
  colbacktitle=gray!55!black,
  coltitle=white,
  fonttitle=\bfseries\small,
  fontupper=\ttfamily\small,
  boxrule=0.5pt,
  arc=2pt,
  left=6pt, right=6pt, top=4pt, bottom=4pt,
}
\newtcolorbox{plainpromptbox}{
  enhanced,
  breakable,
  colback=gray!5,
  colframe=gray!55!black,
  fontupper=\ttfamily\small,
  boxrule=0.5pt,
  arc=2pt,
  left=6pt, right=6pt, top=4pt, bottom=4pt,
}

\title{Do Semantic Distance Tests \emph{Actually} Predict Creativity in LLMs? Evidence for Jagged Creativity}

\author{%
Samuel Schapiro$^{1,2}$\thanks{Correspondence to \texttt{schapironietzsche@gmail.com}.}\quad
Alexi Gladstone$^{3}$\quad
Jonah Black$^{3}$\quad
Heng Ji$^{3}$\\[4pt]
{\normalfont $^{1}$AI Innovation Institute, Stony Brook University\qquad $^{2}$Knowledge Lab, UChicago}\\
{\normalfont $^{3}$University of Illinois, Urbana-Champaign}
}

\begin{document}

\maketitle

\begin{center}
\vspace{-12pt}
\includegraphics[width=0.75\textwidth]{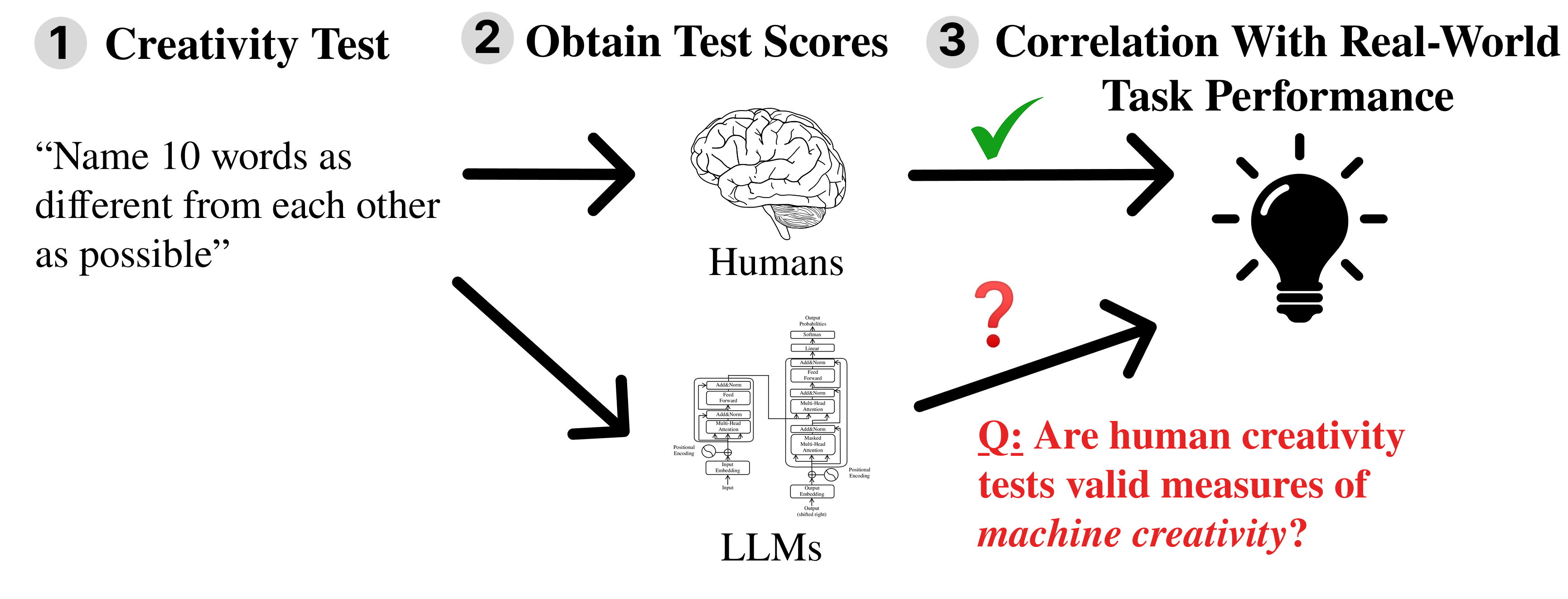}
\captionof{figure}{\textbf{Overview.} Semantic distance tests like the Divergent Association Task ask respondents to name a set of maximally dissimilar words and score the result by embedding  distance. The external validity of these tests among humans has been supported by modest correlations with ideation tasks, but their validity as measures of \emph{machine} creativity has not been adequately established, motivating the central question of this work: \emph{Do semantic distance tests actually predict creativity in LLMs?}}
\label{fig:overview}
\end{center}

\begin{abstract}
Measuring the creativity of large language models (LLMs) is essential for designing methods that can improve creativity and for enhancing our scientific understanding of this ability. To accomplish this, it has become common in recent years to administer semantic distance tests of human creativity to LLMs. Although these tests provide a convenient and fully automated way to score aspects of ``creativity,'' their validity as measures of \emph{machine} creativity has not been established. To address this problem, we conduct the first large-scale, systematic study assessing the effectiveness of semantic distance tests for predicting the creative achievement of LLMs across four target constructs: (a) creative writing, (b) output diversity on open-ended prompts, (c) combinatorial creativity, and (d) scientific ideation. We score each test on two criteria (\emph{validity} and \emph{specificity}), and derive a theoretical limit for the maximum attainable specificity and validity a test can achieve. We find that: \textbf{(1)} No test is a significant predictor of output diversity, suggesting that \emph{point creativity}, the ability to produce a single response connecting semantically distant concepts, and \emph{distributional creativity}, the ability to produce distinct responses across repeated samples, may dissociate in current LLMs---a pattern we call \textbf{\emph{jagged creativity}}. \textbf{(2)} The Remote Associates Test is effectively a proxy for general capabilities rather than a \emph{specific} predictor of creativity in LLMs. \textbf{(3)} Lastly, the Divergent Association Task is the best predictor of creative writing, but no single test predicts all constructs well. In summary, our results prescribe caution when applying psychometric tests and tasks of creativity validated among human populations to LLMs.
\end{abstract}

\section{Introduction} \label{sec:intro}

Evaluating the creativity of large language models (LLMs) is essential for developing methods that improve creativity, advancing our scientific understanding of this ability, and ensuring robust deployment of AI in human co-creativity environments. In recent years, it has become common practice to re-purpose psychological assessments of \emph{human} creativity to purportedly evaluate the ``creativity'' of LLMs \cite{Stevenson2022PuttingTest,Chen2023ProbingAssociation,Bellemare-Pepin2024DivergentLLMs}. Broadly, such tests assess one of two mechanisms implicated in the creative process: \textit{divergent thinking} is the ability to generate multiple distinct responses to an open-ended question, while \textit{convergent thinking} is the ability to generate a single response that unifies multiple diverse stimuli \cite{Dietrich2004TheCreativity, Dietrich2019TypesCreativity, bvsr_update} (\Cref{fig:divergent-convergent}). Central to \textbf{divergent thinking} is the ability to associate semantically distant concepts, which has long been considered fundamental to human creativity \cite{Mednick1962TheProcess., bvsr_update, Thagard2012CreativeInvention}, motivating the use of automated semantic distance assessments like the Divergent Association Task (DAT) \cite{Olson2021NamingCreativity,Chen2023ProbingAssociation,cropley2023artificial,Bellemare-Pepin2024DivergentLLMs,wang2025large} to assess whether LLMs are capable of making distant associations. On the DAT, subjects are instructed to generate ten maximally dissimilar nouns, and scores are given by the mean pairwise semantic distance of all word pairs under an embedding model such as GloVe~\cite{pennington2014glove}, providing a convenient way to assess divergent thinking without human raters. 

\begin{wrapfigure}[27]{r}{0.5\textwidth}
\vspace{-12pt}
\centering
\includegraphics[width=0.5\textwidth]{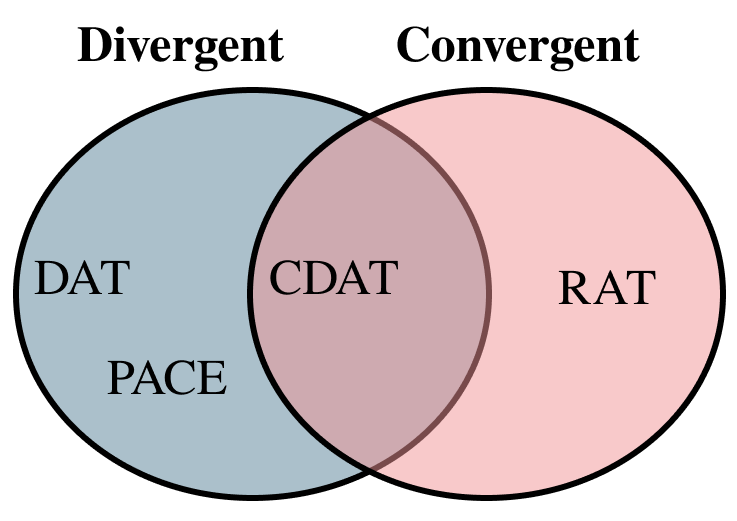}
\caption{\textbf{Semantic distance tests along the divergent--convergent thinking axis.} The Divergent Association Task (\textbf{DAT}) asks for ten maximally dissimilar nouns, and does not involve a convergent thinking component. Similarly, the Parallel Association Chains for Creativity Evaluation (\textbf{PACE}) tests intrinsic associative patterns, and does not involve a convergent thinking component. The Conditional DAT (\textbf{CDAT}) requires each of the maximally dissimilar nouns to be relevant to a cue word, jointly assessing an aspect of convergent thinking. Lastly, the Remote Associates Test (\textbf{RAT}) is a convergent thinking task with a single correct answer per item.}
\label{fig:divergent-convergent}
\end{wrapfigure}
Over the last year, two novel semantic distance tests have also been proposed. The Conditional DAT (CDAT; \cite{nakajima2026beyond}) extends the DAT by requiring dissimilar nouns to be relevant to a ``cue'' word, incorporating a measure of appropriateness in addition to novelty \cite{Boden2004TheMechanisms, Maher2010EvaluatingSystems}. Similarly, the Parallel Association Chain Evaluation (PACE; \cite{pace}) is inspired by the forward-flow measure used in human psychometrics \cite{Gray2019ForwardCreativity., Beaty2021ForwardThinking} and instructs models to make free associations starting from several seed words, scoring responses by a sequential semantic distance measure.

\afterpage{\begin{figure*}[!t]
\definecolor{exDAT}{HTML}{ABC0CB}\definecolor{exCDAT}{HTML}{CEAAB0}\definecolor{exPACE}{HTML}{ABC0CB}\definecolor{exRAT}{HTML}{F8C9CA}
\centering
\footnotesize
\fbox{
\begin{minipage}{0.97\textwidth}\raggedright
\vspace{3pt}
\textbf{Divergent Association Task (DAT).} Generate 10 nouns maximally different from each other.\\[3pt]
\colorbox{exDAT}{\parbox{\dimexpr\linewidth-2\fboxsep\relax}{\raggedright \textit{Response:} ocean, mathematics, hammer, justice, molecule, symphony, volcano, laughter, friction, taxonomy}}\par\vspace{6pt}
\textbf{Parallel Association Chain Evaluation (PACE).} Starting with the seed \textbf{``rock''}, produce three 20-word association chains in which each word associates only with the previous word.\\[3pt]
\colorbox{exPACE}{\parbox{\dimexpr\linewidth-2\fboxsep\relax}{\raggedright \textit{Response (chain 1 of 3):} rock $\to$ stone $\to$ pebble $\to$ beach $\to$ sand $\to$ hourglass $\to$ time $\to$ clock $\to$ alarm $\to$ fire $\to$ smoke $\to$ cigarette $\to$ tobacco $\to$ leaf $\to$ tree $\to$ bark $\to$ dog $\to$ collar $\to$ shirt}}\par\vspace{6pt}
\textbf{Conditional Divergent Association Task (CDAT).} Generate 10 nouns associated with \textbf{``rock''} that are maximally different from each other.\\[3pt]
\colorbox{exCDAT}{\parbox{\dimexpr\linewidth-2\fboxsep\relax}{\raggedright \textit{Response:} stone, guitar, music, geology, cliff, mineral, foundation, cradle, concert, pebble}}\par\vspace{6pt}
\textbf{Remote Associates Test (RAT).} What single word can be combined with each of \textbf{``cream''}, \textbf{``skate''}, and \textbf{``water''} to form a compound word or common phrase?\\[3pt]
\colorbox{exRAT}{\parbox{\dimexpr\linewidth-2\fboxsep\relax}{\raggedright \textit{Response:} ice}}
\vspace{3pt}
\end{minipage}}
\caption{\textbf{Example prompts and responses for each semantic distance test.} The DAT prompts for maximally distant nouns, PACE prompts models to freely associate sequences of nouns in a 20-word chain, the CDAT prompts for maximally distant nouns that are each relevant to a cue, and the RAT prompts for the single word that links three cue words.}
\label{fig:examples}
\end{figure*}
}
\textbf{Convergent thinking}, on the other hand, is commonly assessed by the Remote Associates Test (RAT) \cite{Mednick1962TheProcess., bowden2003normative}, which presents three stimulus words, and asks for a single word that connects all three (e.g., given \emph{cottage, swiss}, and \emph{cake}, the response ``cheese'' applies to each stimulus). Various other convergent creativity assessments and benchmarks for LLMs have been introduced in recent years, including the Only Connect Wall test and the CresOWLve benchmark \cite{Naeini2023LargeDataset,cresowlve}, and, for multimodal LLMs, the Labyrinth of Links \cite{li2025labyrinth} and MM-OPERA \cite{huang2025mmopera} benchmarks.

Despite their growing use, it remains unclear what these tests measure in LLMs. For example, PACE shows strong Spearman rank correlations with creative writing benchmarks ($\rho \approx 0.74$). However, it also correlates strongly with general model capabilities ($\rho \approx 0.66$), so it is unclear how well PACE measures creative achievement \emph{independent} of what general capability already predicts. More broadly, the external validity of semantic distance tests (i.e., as predictors of creativity in LLMs on real-world tasks) has not been established, and their validity as predictors of human creativity remains limited(see \Cref{sec:motivating_problems}).

\paragraph{Our Contributions}
To address these problems, we carry out the first systematic study assessing the effectiveness of semantic distance tests for predicting creative achievement in LLMs. We measure creative achievement using seven benchmarks that capture four target constructs: (i) creative writing, (ii) output diversity on open-ended prompts, (iii) combinatorial creativity, and (iv) scientific ideation. We also introduce two evaluation criteria: \emph{validity}, which measures the raw correlation $r$ with each benchmark, and \emph{specificity}, the semi-partial correlation $r|g$ after residualizing benchmark scores on a capability proxy $g$. The latter assesses how well semantic distance tests predict aspects of creative achievement \emph{after} variance explained by general capabilities is factored out, which is particularly important given that correlations between benchmarks and general capabilities are as high as $r = 0.87$ (\Cref{fig:benchmark-correlations}a). In detail, we make the following contributions and findings:
\begin{enumerate}
    \item We conduct a large-scale, systematic study evaluating the effectiveness of semantic distance tests for predicting the creative achievement of LLMs.
    \item We introduce evaluation metrics which measure a test's predictive power, both in raw correlations with benchmarks (construct \textit{validity}) and independent of what general capabilities already predict (\textit{specificity}), and bound the validity and specificity a test can jointly achieve.
    \item We find that no test significantly predicts output diversity, including the DAT, suggesting that \emph{point creativity} (producing a single response that connects semantically distant concepts) and \emph{distributional creativity} (producing distinct responses across repeated samples) may dissociate in LLMs, a pattern we call \textbf{\emph{jagged creativity}} (\Cref{fig:jagged}) by analogy to the \emph{jagged intelligence} of modern AI systems \cite{mineault2026cognitive}. For example, \texttt{gpt-5-4}, the top DAT scorer, includes ``quasar'' in 39 of its 40 DAT responses (\Cref{tab:dat-repetition}).
    \item We find that the Remote Associates Test (RAT) is effectively a capability proxy---its significant correlations with creative writing benchmarks become non-significance once general capabilities are statistically controlled for.
    \item Our empirical findings indicate that the Divergent Association Task (DAT) is the best predictor of creative writing, but that no single test predicts all constructs well.
\end{enumerate}

\section{Background and Motivating Problems} \label{sec:motivating_problems}
\begin{notebox}{Note on Terminology}
Creativity can refer to a person, process, press (environment) or product \cite{Rhodes1961Analysis}, and in this work, we mainly focus on creative products. Furthermore, we distinguish between a model's \emph{score on a creativity test} (like the DAT) from its \emph{creative achievement}, the latter of which refers to the novelty and appropriateness of the artifacts it actually produces on real-world tasks assessed by creativity benchmarks \cite{Boden2004TheMechanisms, Runco2012Standard}.
\end{notebox}

In recent years, semantic distance tests have been widely adopted to assess whether LLMs are capable of making distant associations \cite{Chen2023ProbingAssociation, cropley2023artificial, Bellemare-Pepin2024DivergentLLMs}, which have long been thought to underlie creative ideation \cite{Mednick1962TheProcess.}. Furthermore, various strategies, such as varying sampling parameters, instructing LLMs to assume the personas of distinguished creative individuals \cite{wang2025large}, prompt engineering \cite{wang2025large}, offering explicit test-taking strategies \cite{Bellemare-Pepin2024DivergentLLMs}, and steering model weights \cite{schapiro2026creativityneuro}, have all been explored to improve semantic distance test scores. Additionally, semantic distance tests are already used to make direct claims that LLMs are \emph{more} or \emph{less} creative than humans~\cite{Chen2023ProbingAssociation,Bellemare-Pepin2024DivergentLLMs,wang2025large}, which presupposes that the tests measure the same construct in both populations. However, this is problematic for several reasons.

\begin{figure}[!tb]
\centering
\begin{minipage}[b]{0.53\textwidth}
\footnotesize
\hrule height 0.8pt\vspace{3pt}
\begin{algorithmic}[1]
\REQUIRE vocabulary $V$, embedding $E$, number of words $n$ (= 10 for the DAT)
\ENSURE word list $W = (w_1, \dots, w_n)$
\STATE $w_1 \sim \mathrm{Uniform}(V)$
\STATE $W \leftarrow [w_1]$
\FOR{$i = 2$ \dots $n$}
    \STATE $w_i \leftarrow \arg\min_{v \in V \setminus W} \tfrac{1}{|W|} \sum_{u \in W} \cos(E(v), E(u))$
    \STATE $W \leftarrow W \cup \{w_i\}$
\ENDFOR
\STATE \textbf{return} $W$
\end{algorithmic}
\vspace{2pt}\hrule height 0.8pt
\vspace{4pt}
\subcaption{Greedy algorithm}\label{fig:greedy-alg}
\end{minipage}\hfill
\begin{minipage}[b]{0.44\textwidth}
\centering
\includegraphics[width=\linewidth]{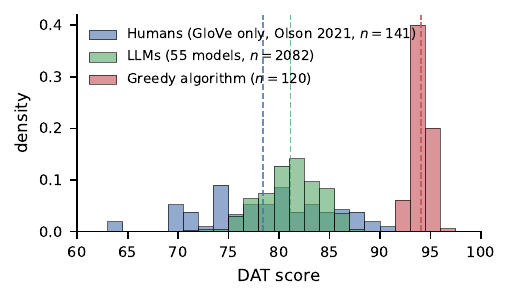}
\subcaption{DAT scores}\label{fig:greedy-hist}
\end{minipage}
\caption{\textbf{A simple greedy algorithm exceeds mean human and LLM DAT scores.} \textbf{(a)} The algorithm. \textbf{(b)} DAT score distributions; dashed lines mark group means (details in \Cref{app:greedy}).}
\label{fig:greedy-baseline}
\end{figure}

For one, \textbf{the validity of semantic distance tests as measures of \emph{machine} creativity has not been established.} Measures of human intelligence are often inappropriate to administer naively to LLMs~\cite{chollet2019measure} because such tests do not inherently possess construct validity for machines.\footnote{The digit span test is commonly used to assess human intelligence \cite{iq_tests}, but would be inappropriate to administer to machines, where working memory is not a bottleneck to intelligence the way it is among humans.} Analogously, semantic distance tests can be ``hackable'' by computational mechanisms that would not be considered ``creative'' \cite{Boden2004TheMechanisms}. Consider that Transformers have been shown to implement algorithmic circuits in their weights \cite{olsson2022induction,nanda2023grokking}, and that the DAT itself can be solved by a simple greedy algorithm over embedding distances (\Cref{fig:greedy-alg}) whose scores substantially exceed the mean human and LLM scores ($94.1$ vs.\ $78.4$ and $81.1$; \Cref{fig:greedy-hist}). A psychometric test is a valid measure of creativity if scores correlate strongly with quantitative measures of creative achievement \emph{within that same population}. Semantic distance tests were designed for human populations, and their external validity as measures of \emph{machine} creativity has not yet been established.\footnote{Moreover, these tests may be flawed due to training data leakage, as has been observed with the AUT \cite{Stevenson2022PuttingTest}}

Furthermore, \textbf{the validity of semantic distance tests as measures of \emph{human} creativity is itself only modestly established.} Even within human populations, the external validity of these tests rests on a short and modest chain of correlations. Both the DAT \cite{Olson2021NamingCreativity} and forward flow\footnote{PACE \cite{pace} uses the same scoring mechanism as forward flow.} ~\cite{Gray2019ForwardCreativity.} are validated by correlation with the Alternative Uses Test (AUT): the DAT at $r \approx .32$--$.51$~\cite{Olson2021NamingCreativity}\footnote{The DAT is also validated against the Bridge-the-Associative-Gap Test \cite{bag_test}, a measure of convergent thinking.} and forward flow at $r \approx .43$--$.49$~\cite{Beaty2021ForwardThinking}. However, meta-analyses have revealed that AUT scores themselves correlate with self-reported creative achievement at only $r \approx .17$ \cite{said2024divergent} to $r \approx .22$~\cite{kim2008meta}, and \cite{baer2011ttct} reports near-zero correlations between AUT scores and creative achievement, casting doubt on the external validity of semantic distance tests in predicting creative outcomes among humans.

\begin{figure*}[!tb]
\centering
\includegraphics[width=0.85\textwidth]{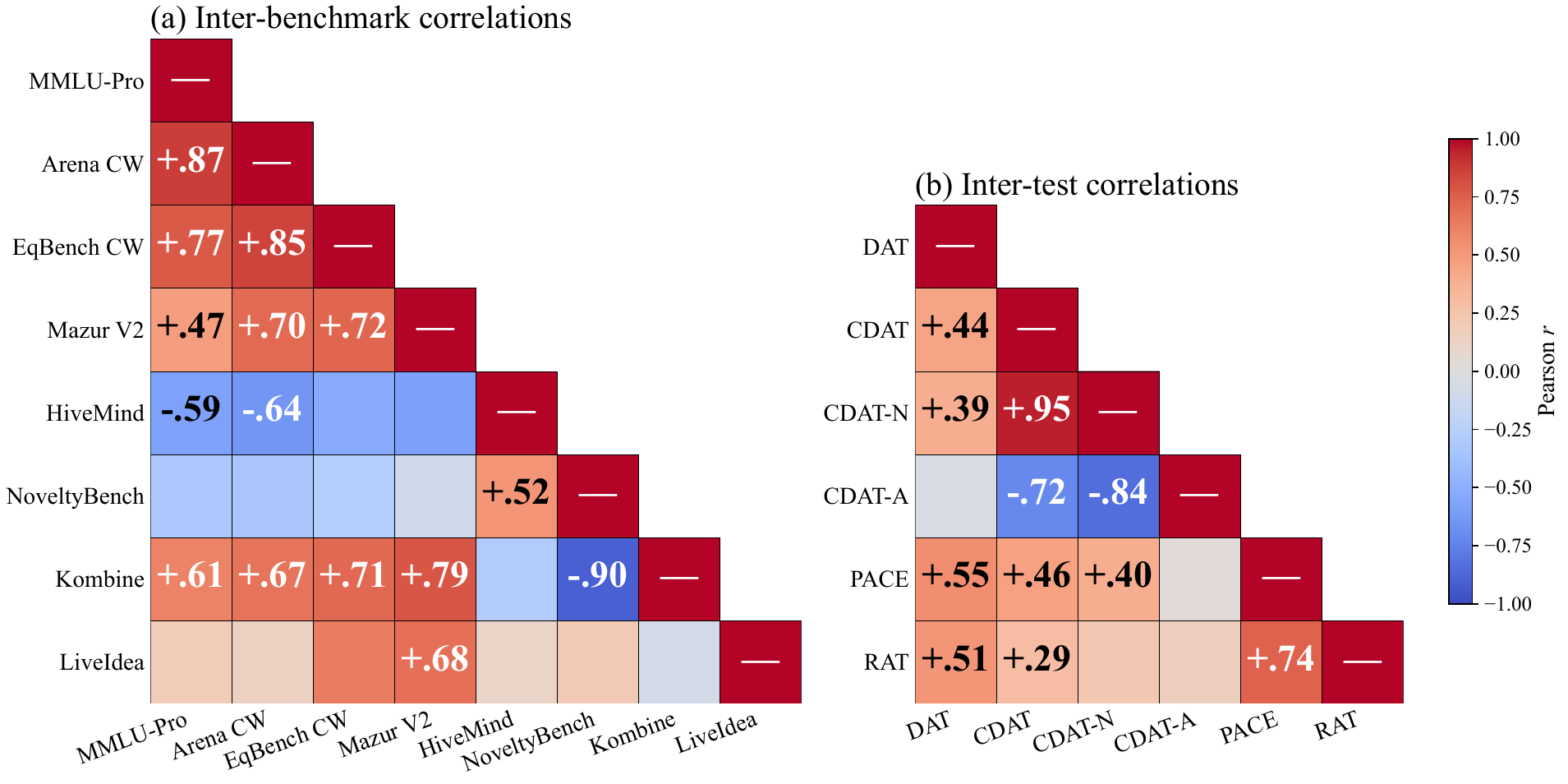}
\caption{\textbf{Inter-benchmark and inter-test correlations.} \textbf{(a)} Benchmarks are ordered by construct: general capability (MMLU-Pro), creative writing (Arena CW, EqBench CW, Mazur CW), output diversity (Hivemind, NoveltyBench), combinatorial creativity (Kombine), and scientific ideation (LiveIdea). Each cell uses the models scored on both benchmarks, from $n = 6$ (Kombine~$\times$~Hivemind) to $n = 40$ (MMLU-Pro~$\times$~Arena CW). \textbf{(b)} Inter-test correlations among DAT, CDAT, CDAT-N, CDAT-A, PACE, and RAT ($n = 64$--$68$), where each embedding-scored test uses its mean $z$-score across the GloVe, FastText, and SBERT embeddings. Values are shown only where $p < .05$.}
\label{fig:benchmark-correlations}
\end{figure*}

Finally, we argue that \textbf{semantic distance tests should measure something independent of what general capability already predicts.} A semantic distance test that correlates with, e.g., creative writing rankings, does not, on its own, establish that the test measures ``creativity.'' A semantic distance test should predict aspects of creative achievement in ways that are statistically \emph{independent of} what general capability already predicts---otherwise it reduces to a measure of general capability rather than creativity. Among humans, where divergent thinking test scores are known to correlate with general intelligence up to $r = 0.37$ \cite{gerwig2021relationship}, the standard practice is to factor out confounding variables (e.g., through hierarchical regression or factor analysis) before claiming a test measures ``creativity'' \cite{Beaty2021ForwardThinking}. In \Cref{sec:method}, we propose a similar methodology for measuring the predictive power of semantic distance tests independent of what general capabilities already predict.

\section{Preliminaries} \label{sec:background}

With these problems in mind, we now introduce the tests and benchmarks we evaluate.

\subsection{Semantic Distance Tests}

\paragraph{Divergent Association Task (DAT)}
Introduced by \cite{Olson2021NamingCreativity}, the DAT asks subjects to name 10 nouns\footnote{Usually, only the first seven valid nouns are scored.} $W = \{w_1, \dots, w_n\}$ as different from each other as possible, and scores the average cosine distance between all word pairs:
\begin{equation} \label{eq:dat}
    \textrm{DAT}_\mathcal{E}(W) := \tfrac{100}{n(n-1)} \sum_{i \neq j}^{N} (1 - \cos_\mathcal{E}(\mathbf{w}_i, \mathbf{w}_j))
\end{equation}
In addition to computing scores under the standard GloVe 840B embeddings~\cite{pennington2014glove} used in \cite{Olson2021NamingCreativity}, we test robustness under multiple embedding models in~\Cref{sec:method}.

\paragraph{Conditional Divergent Association Task (CDAT)} Creative artifacts should be both novel \emph{and} useful \cite{Varshney2019MathematicalCreativity,Boden2004TheMechanisms,Simonton2004CreativityZeitgeist,Maher2010EvaluatingSystems}. \cite{nakajima2026beyond} argued that the original DAT failed to measure utility, and proposed the conditional DAT (CDAT), in which a cue word $c$ is added, and each word must be sufficiently relevant to the cue word in order for its novelty to count. For words $W = \{w_1, \dots, w_n\}$ generated for cue $c$, novelty and appropriateness are measured according to:
\begin{align} \label{eq:cdat-novelty}
    \textrm{CDAT-N}_\mathcal{E}(W) := \tfrac{100}{n(n-1)} \sum_{i\neq j}^{N} (1 - \cos_\mathcal{E}(\mathbf{w}_i, \mathbf{w}_j)); &&&
    \textrm{CDAT-A}_\mathcal{E}(W \mid c) := \tfrac{100}{n} \sum_{i=1}^{n} \cos_\mathcal{E}(\mathbf{c}, \mathbf{w}_i)
\end{align}
Note that CDAT-N uses Equation~\eqref{eq:dat}, but its scores differ from the DAT's due to cue-based prompting.

\paragraph{Parallel Association Chain Evaluation (PACE)}
PACE~\cite{pace}, inspired by forward-flow associative-chain methods~\cite{Gray2019ForwardCreativity.}, prompts models to produce three parallel 20-word association chains from a seed, scored by the mean cumulative cosine distance along each chain $C = (w_1, \dots, w_L)$:
\begin{equation}
    \textrm{PACE}_\mathcal{E}(C) := \tfrac{1}{L-1} \sum_{i=2}^{L} \tfrac{1}{i-1} \sum_{j=1}^{i-1} (1 - \cos_\mathcal{E}(\mathbf{w}_i, \mathbf{w}_j))
\end{equation}
The model-level score is the mean of $\textrm{PACE}_\mathcal{E}(C)$ across three chains per seed and across all seeds. In the original paper, PACE correlated with Chatbot Arena CW at $\rho \approx 0.74$ with $n=30$. The metric was originally tested under FastText embeddings~\cite{mikolov2018fasttext}, although \cite{pace} also saw significant rank correlations with Chatbot Arena CW under multiple distinct embedding models.

\paragraph{Remote Associates Test (RAT)}
The Remote Associates Test~\cite{Mednick1962TheProcess.,bowden2003normative} is a long-standing psychometric instrument for assessing the convergent thinking component of creativity. The RAT presents three remote stimulus words and asks for the single word that connects all three. For example, given $(\text{cottage}, \text{swiss}, \text{cake})$, the answer is \emph{cheese} (\emph{cottage cheese}, \emph{swiss cheese}, \emph{cheesecake}). Unlike previous tests, where scoring is based on embedding distance, the RAT scores items as correct or incorrect against an answer key. We use a 30-item subset of the \cite{bowden2003normative} stimulus set and score each model's ability to produce the correct response in a zero-shot fashion.

\subsection{Creative Achievement Benchmarks}
We evaluate each test against four target constructs (creative writing, output diversity on open-ended prompts, combinatorial creativity, and scientific ideation) spanning seven benchmarks.

\subsubsection{Creative Writing}

\textbf{Chatbot Arena Creative Writing (Arena CW)}~\cite{chiang2024chatbot}. Arena CW is a creative writing benchmark where models respond to open-ended, user-submitted creative writing prompts, and pairwise user preferences are aggregated into per-model Elo ratings. Arena CW is the most capability-loaded of the three creative writing benchmarks ($r = 0.87$ with MMLU-Pro; \Cref{fig:benchmark-correlations}).

\textbf{EQ-Bench Creative Writing}~\cite{paech2023eqbench}. EQ-Bench is a creative writing benchmark where models produce responses to a set of 32 distinct prompts designed to challenge models in areas like humor, romance, spatial awareness, and unique perspectives. Responses are scored head-to-head by Claude Sonnet using a 9-criterion rubric and aggregated into Elo ratings, and the benchmark incorporates strategies to reduce known LLM-as-a-judge biases (length, position, poetic incoherence, etc.). EQ-Bench correlates strongly with MMLU-Pro ($r = 0.77$; \Cref{fig:benchmark-correlations}), although less than Arena CW.

\textbf{Mazur Creative Writing}~\cite{mazur2025writing}.
Mazur Creative Writing is a creative writing benchmark that tests how well LLMs incorporate a set of mandatory story elements (characters, objects, core concepts, attributes, motivations, etc.) in a short creative story. An ensemble of LLM judges applies an $\sim$18-question rubric assessing whether mandatory elements are adequately represented, as well as the overall narrative quality. The final story score is the mean of the per‑grader scores. Mazur Creative Writing correlates moderately with MMLU-Pro ($r = 0.47$), less than Arena CW and EQ-Bench.

\subsubsection{Output Diversity}
We use two benchmarks that measure output diversity in LLM responses to open-ended prompts.

\textbf{Hivemind}~\cite{hivemind}. Hivemind is an output diversity benchmark which measures output diversity across open-ended queries,  drawn from real-world user-ChatGPT interactions. Categories include brainstorm and ideation (``Suggest a feature for a smartwatch designed specifically for senior citizens.''), philosophical questions (``How do I understand what I want?''), speculative and hypothetical scenarios (``Create a short review of a future movie.''), and ambiguous everyday questions (``How can I live on \$1,000 per month?''). Scoring measures intra-model repetition via average pairwise embeddings dissimilarity\footnote{The original paper uses similarity, but for consistency with the output diversity construct, we report its complement.} using OpenAI's \texttt{text-embedding-3-small} model. Hivemind scores are negatively correlated with capability ($r = -0.59$ with MMLU-Pro).

\textbf{NoveltyBench}~\cite{zhang2025noveltybench}. NoveltyBench evaluates the output diversity of LLMs on a similar collection of real-world user-ChatGPT interactions, as well as open-ended prompts across four categories: randomness (``Roll a make-believe 20-sided die''), factual information (``List a capital city in Africa.''), creative text writing (``Tell me a riddle.''), and subjectivity (``What's the best car to get in 2023?''). Scoring assesses the extent to which a sequence of generations is both diverse and high-quality, using a cumulative utility measure that penalizes functional equivalence to prior responses. NoveltyBench is only weakly correlated with capability, obtaining $r = -0.32$ with MMLU-Pro.

\subsubsection{Combinatorial Creativity}

\textbf{Kombine}~\cite{schapiro2026measuringcombinatorialcreativitylarge}.  Kombine is a combinatorial creativity benchmark tasking models to discover surprising, original, and useful associations, analogies, and blends between unrelated entities drawn from the Wikidata knowledge graph, and to use those analogies and blends to invent novel concepts by re-combining the parts of existing ones. Responses are scored by embedding-based measures such as surprise and originality, as well as judge-based measures on utility and emergent creativity. Here, we use each model's overall score across the three tasks. Kombine correlates strongly with general capabilities ($r = 0.61$ with MMLU-Pro; \Cref{fig:benchmark-correlations}a).

\subsubsection{Scientific Ideation}

\textbf{LiveIdeaBench}~\cite{ruan2024liveideabench}. LiveIdeaBench assesses LLMs on open-ended scientific idea generation. Given a single keyword (e.g., a domain term) drawn from a set spanning 18 scientific disciplines, a model must propose a research idea relevant to that keyword. Each idea is scored by a panel of LLM judges along five dimensions---\emph{fluency} (does the response actually contain a usable idea), \emph{feasibility}, \emph{clarity}, \emph{originality}, and \emph{flexibility}---with overall scores reported as the average across all dimensions. The original benchmark explicitly demonstrates that scientific ideation can dissociate from general capability (e.g., QwQ-32B-preview outperforms much larger frontier models on Originality despite being weaker on knowledge tests), and empirically, we observe a similar trend where LiveIdeaBench is only weakly correlated with general capabilities (\Cref{fig:benchmark-correlations}a; $r = 0.18$ on MMLU-Pro).

\section{Evaluation Method} \label{sec:method}

\subsection{Models and Inference}
In order to evaluate the effectiveness of semantic distance tests for predicting creative achievement benchmarks, we obtain raw test scores across a large set ($n=55$) of instruction-tuned LLMs across 10 providers (OpenAI, Anthropic, Google, Meta, Mistral, Qwen, DeepSeek, Cohere, NVIDIA, Microsoft) via OpenRouter. DAT and CDAT use $T \in \{1.0, 1.5, 2.0\}$ (40 trials per temperature, with \texttt{top\_p}$=1$, \texttt{top\_k}$=0$). PACE follows the original setup in \cite{pace}, where sampling is done with temperature $T = 0$ and stochasticity is controlled by varying the random seed for each anchor word. We use 50 anchor words, and collect three parallel 20-word chains per seed. Lastly, on the CDAT, we use 50 cue words, and following \cite{nakajima2026beyond}, per-cue appropriateness values are compared to a random-noun baseline via Welch's $t$-test, with a Benjamini--Hochberg false discovery rate (FDR) correction at $\alpha = .001$ across models within each temperature. A model's responses at a given temperature are retained if their FDR-adjusted $p$-value passes and its mean appropriateness exceeds the random baseline. The CDAT score we report is the mean of CDAT-N across passing temperatures. For reproducibility purposes, exact prompts used for each test are given in Appendix~\ref{app:prompts}.

\paragraph{Embedding models} Each semantic distance test score depends on the embedding model used for pairwise cosine similarity. To evaluate robustness to embedding choice, we score each test under three embeddings---GloVe 840B 300d, FastText crawl-300d-2M, and Sentence-BERT \texttt{all-mpnet-base-v2}---which differ in training. GloVe and FastText are static word-level co-occurrence models (300-dim), while Sentence-BERT (\texttt{all-mpnet-base-v2}; 768-dim) is a transformer sentence encoder trained via contrastive sentence-pair objectives.

\begin{figure*}[t]
\centering
\includegraphics[width=\textwidth]{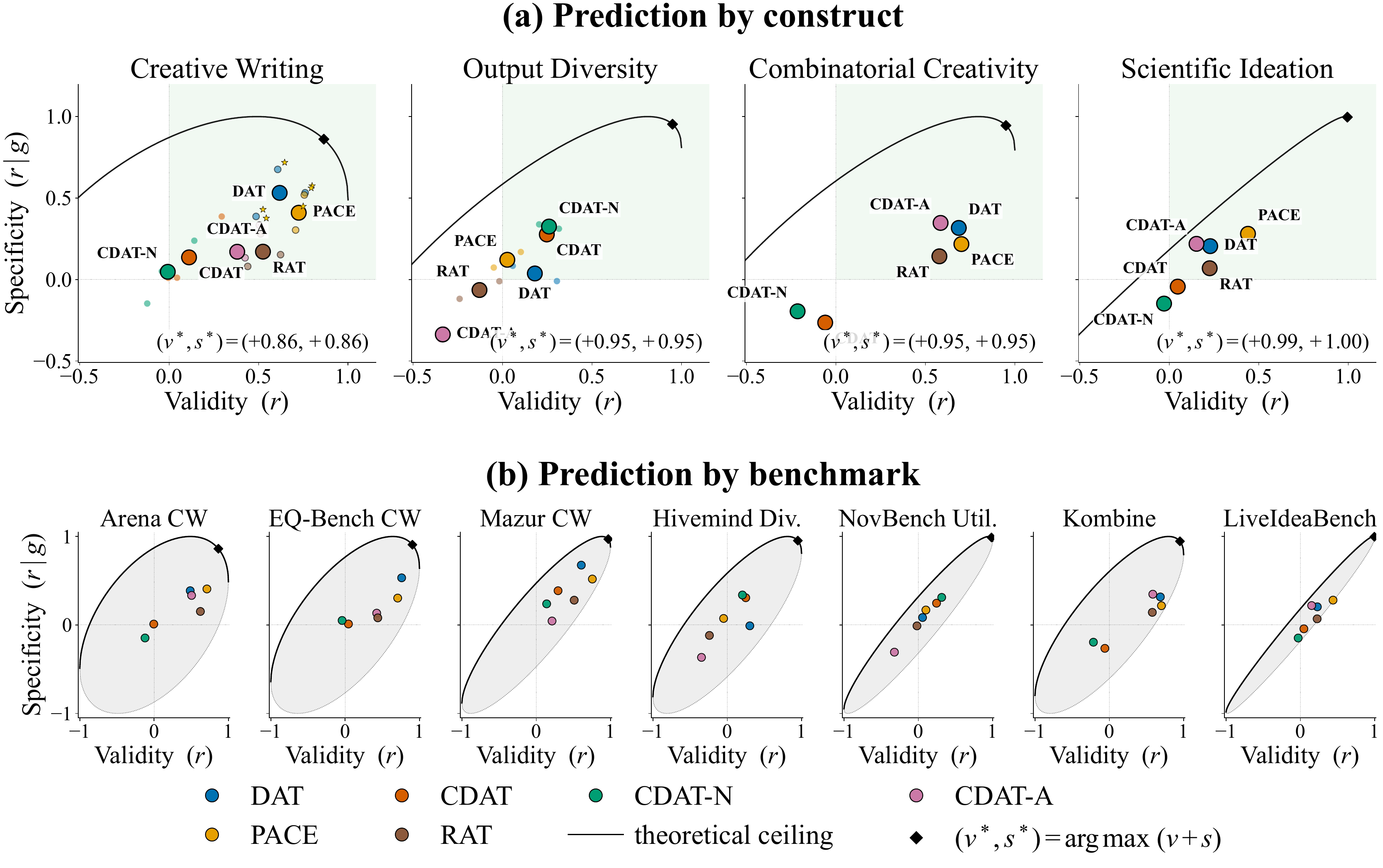}
\caption{\textbf{Validity and specificity by construct and benchmark, with theoretical ceilings.} \emph{(a) Prediction by construct:} Each small  point is a (test, benchmark) cell from the Overall block of Table~\ref{tab:correlations}, while each large black-outlined circle is a test's average across the construct's benchmarks. A gold $\bigstar$ = positive and significant on both axes ($p < .05$). The black curve is the construct-level theoretical specificity ceiling obtained in \Cref{thm:spec-ceiling}, and the black diamond marks its maximum of $v + s$. \emph{(b) Prediction by benchmark:} the same ceiling for each benchmark. The grey region is the feasible $(v, r(X,Y\mid g))$ set.}
\label{fig:headline}
\vspace{-10pt}
\end{figure*}

\subsection{Evaluation Metrics}
We evaluate each test on two criteria, validity and specificity, each motivated and defined below.

\textbf{Validity} \space \space Standard psychometric criteria used in human creativity research \cite{Beaty2021ForwardThinking, Olson2021NamingCreativity} indicate that a test has construct validity if its scores correlate with external measures of that construct. Therefore, we report \emph{validity} as the raw Pearson correlation $r(X, Y)$ between test score $X$ and benchmark score $Y$.

\textbf{Specificity} \space \space Since creative achievement benchmarks are themselves highly capability-loaded (for example, Arena CW correlates with MMLU-Pro at $r = 0.87$) a high raw $r(X, Y)$ may simply reflect that the test predicts general capability. Thus, we report \emph{specificity}, the semi-partial Pearson correlation between the test score $X$ and the benchmark $Y$ residualized on a capability stack $g$, given by $r(X, Y \mid g) \;:=\; r\!\Bigl(X,\; Y - \hat{Y}_{g}\Bigr)$. Here, $\hat{Y}_{g}$ is the ordinary least squares prediction of $Y$ from MMLU-Pro accuracy \cite{wang2024mmlupro}. The semi-partial correlation measures to what extent the test predicts the benchmark's capability-controlled variance, with a significant positive $r(X, Y \mid g)$ indicating the test predicts variance that capability cannot already explain.

For both validity and specificity, we report two-sided $p$-values from the standard Pearson $t$-test, $t = r\sqrt{(n - 2 - k)/(1 - r^2)} \sim t_{n - 2 - k}$, where $k$ is the number of controls regressed out ($k = 0$ for validity and $k = 1$ for MMLU-Pro used for specificity). 

\subsection{Theoretical Limits for Attainable Validity and Specificity} \label{subsec:val_spec_frontier}
What is the maximum specificity a test can achieve, given its validity and its benchmark's correlation with general capabilities? We call this the \emph{validity-specificity frontier} and bound it below.

\begin{theorem}[Validity-Specificity Frontier]
\label{thm:spec-ceiling}
Let $g = (Z_1, \dots, Z_k)$ be a vector of general capability proxies, let $\hat{Y}_g$ denote the ordinary least squares prediction of $Y$ from $g$, and define $R = \mathrm{corr}(Y, \hat{Y}_g)$. For any test $X$ with validity $v = \mathrm{corr}(X, Y)$, the semi-partial correlation $r(X, Y \mid g) := r(X,\, Y - \hat{Y}_g)$ satisfies
\begin{equation}
    |r(X, Y \mid g)| \;\leq\; v\sqrt{1 - R^2} \,+\, |R|\,\sqrt{1 - v^2}.
    \label{eq:spec-ceiling}
\end{equation}
\end{theorem}

\noindent The proof of \Cref{thm:spec-ceiling} is given in \Cref{app:spec-bound} of the appendix. The intuition is that a test highly correlated with a capability-loaded benchmark \emph{cannot} be very decoupled from capability (i.e., specificity). In particular, a perfectly valid test ($v = 1$) has specificity bounded by $\sqrt{1 - R^2}$, and a benchmark with $|R|$ near $1$ therefore leaves almost no specificity headroom for \emph{any} test, regardless of construction. This is why Arena CW ($R = 0.87$ on $g = \text{MMLU-Pro}$) caps specificity at $\approx 0.49$ for a perfect-validity test, while NoveltyBench Utility ($R \approx -0.32$) admits a wider ceiling (\Cref{fig:headline}b).

\section{Results} \label{sec:results}

\Cref{tab:correlations} (\Cref{app:correlations}) reports validity and specificity for all test--benchmark pairs. Test and benchmark scores used for correlations are fully reported in \Cref{tab:per-model-scores} and \Cref{tab:per-model-benchmarks}, respectively. \Cref{fig:headline} plots validity and specificity across the four constructs against the validity-specificity frontier (\Cref{subsec:val_spec_frontier}).

\begin{wrapfigure}{r}{0.42\textwidth}
\vspace{-12pt}
\centering
\includegraphics[width=0.40\textwidth]{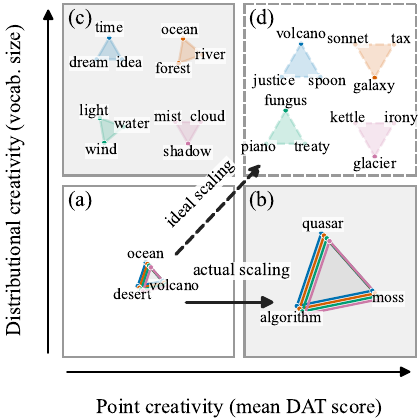}
\caption{\textbf{Jagged creativity.} In (a)--(c), each triangle is three of the ten words from one real DAT response of (a) Mistral 7B, (b) GPT-5.4, and (c) o3-mini, and colors are repeated samples; (d) is hypothetical. (b) and (c) are jagged, while (a) and (d) reflect non-jagged intelligence in which distributional and point creativity scale together, as would be expected among humans \cite{Mednick1962TheProcess.}. As shown in \Cref{tab:base-vs-instruct}, while instruction tuning raises the mean DAT score from 77.4 to 81.4, it cuts the number of distinct words a model uses by about two-thirds, from 296 to 97 on average. \protect\footnotemark}
\label{fig:jagged}
\vspace{-10pt}
\end{wrapfigure}
\footnotetext{This distinction parallels \cite{Mednick1962TheProcess.}'s distinction between \emph{flat} and \emph{steep} associative hierarchies, the original distinction offered among human populations: individuals with flat hierarchies have many remote associates of similar strength to a given stimulus, whereas individuals with steep hierarchies have a few strong, stereotyped associates. Mednick further noted that a \emph{steep-deviant} individual can still be creative if their typical associates are deviant from the population's, as in (b).}
\textbf{None of the tests is a significant predictor of output diversity.} \space \space On the two output diversity benchmarks, no test achieves significant validity or specificity. Notably, this includes the DAT (Hivemind $r = +0.30$, $r|g = -0.01$; NoveltyBench Utility $r = +0.06$, $r|g = +0.08$), even though the DAT explicitly asks for words that are as different from each other as possible. This suggests that the ability to generate a single, highly divergent DAT or CDAT response, whether through memorization or an algorithm like the one in \Cref{fig:greedy-alg}, is distinct from the ability to continually generate distinct responses to open-ended questions. As a result, we find no evidence that \emph{point creativity}, the ability to produce a single creative response that associates more semantically distant entities, predicts \emph{distributional creativity}, the ability to produce a diverse set of responses across repeated samples, suggesting the two may dissociate in current frontier LLMs, a pattern we call \textbf{\emph{jagged creativity}} (\Cref{fig:jagged}). \Cref{tab:dat-repetition} highlights this \emph{jaggedness} in raw DAT responses: \texttt{gpt-5-4}, the top DAT scorer, includes ``quasar'' in 39 of 40 responses and \texttt{claude-sonnet-4-6} includes ``justice'' in all 40, even at $T = 2.0$. Comparing five base models, given the DAT as a text completion, with their instruction-tuned versions, instruction tuning raises the mean DAT score from 77.4 to 81.4 but cuts the number of distinct words a model uses across 40 responses by about two-thirds, from 296 to 97 on average (\Cref{tab:base-vs-instruct}). 

\textbf{The RAT is not a specific creativity test for LLMs.} \space \space A core motivation addressed in \Cref{sec:motivating_problems} is that tests should measure aspects of creative achievement independent of what general capabilities already predict. On creative writing benchmarks, the RAT obtains statistically significant validity ($r = +0.62^{***}, +0.44^{*}, +0.51^{*}$); however, after controlling for general capabilities, the RAT collapses to non-significant specificity ($r | g \in [+0.08, +0.28]$). The same pattern holds on combinatorial creativity, where the RAT's validity ($r = +0.58^{**}$) drops to $r | g = +0.14$. In summary, the convergent thinking assessed by the RAT largely measures what capability already predicts.

\textbf{The DAT is the best predictor of creative writing.} \space \space The DAT is the only test that achieves significant validity \emph{and} specificity on all three creative writing benchmarks (Arena CW $r = +0.49^{***}$, $r|g = +0.39^{*}$; EQ-Bench CW $r = +0.76^{***}$, $r|g = +0.53^{**}$; Mazur CW $r = +0.61^{**}$, $r|g = +0.68^{**}$). PACE obtains the highest raw validity on creative writing ($r \in [0.71, 0.76]$), but retains significant specificity only on Arena CW ($r|g = +0.41^{*}$) and Mazur CW ($r|g = +0.52^{*}$).

\textbf{No single test predicts all constructs well.} \space \space Although some tests have significant validity on multiple constructs, no test retains significant specificity on more than one construct. For example, on combinatorial creativity, PACE is the most valid predictor ($r = +0.70^{***}$), but not a specific one ($r|g = +0.22$). Similarly, on scientific ideation, PACE is the only test with significant validity ($r = +0.44^{*}$, $n = 25$), but likewise lacks significant specificity.\footnote{We find no evidence of a large or moderate, capability-independent association between any existing test and scientific ideation; at our current coverage ($n = 20$--$22$) the analysis is sensitive only to $|r| \geq 0.56$--$0.58$, so a small association cannot be excluded. However, the same tests detect large associations with creative writing at similar sample sizes, so the absence of significant correlations for scientific ideation indicates a weaker (if any) relationship, not merely low power.} In summary, each test is a good predictor of one specific construct (at best), rather than a general-purpose ``creativity'' measure.

\section{Discussion} \label{sec:discussion}

\textbf{Localizing the algorithms behind test performance} \space \space Our results raise the question of \emph{how} models obtain high scores on semantic distance tests. A high DAT score could reflect a memorized list of distant words, an algorithm like \Cref{fig:greedy-alg} implemented in the model's weights, or a more general capacity for distant association, and only the last should be expected to transfer to creative achievement. 
Mechanistic interpretability---the AI subfield concerned with understanding the relationship between model internals and model behavior \cite{sharkey2025open}---can leverage the structure of semantic distance tests as guides for localizing relevant circuits in model weights and activations that models use on these tasks, analogous to studies that have identified circuits for arithmetic abilities \cite{nanda2023grokking}. Prior work has taken a first step in this direction, localizing the weights that matter for divergent thinking and amplifying them to raise DAT scores, as well as human-judged creativity on longer-form tasks \cite{schapiro2026creativityneuro}.

\textbf{Beyond embedding-based distance measures} \space \space In text-to-image creativity studies, hierarchical concept ontologies that decompose visual concepts (e.g., Sofa) into constituent parts (e.g., leg, cushion) and associated uses (e.g., support, rest) have been proposed to support combinatorial creativity in a more explicit way
\cite{design2025}. Ontology-based measures may better capture hierarchical and logical relations between concepts and can complement semantic distance measures in future work.

\textbf{Testing for transformational creativity}  \space \space
One understudied facet of creativity that has historically played an outsized role in economic innovation, scientific discovery, and artistic achievement \cite{Boden2004TheMechanisms} is transformational creativity, which involves altering existing conceptual spaces in ways that radically transform the scope of conceivable artifacts.\footnote{Einstein's relativity theory and Copernicus's heliocentric theory are two canonical examples \cite{Schapiro2025TransformationalTheory}.} Recent work has begun to explore both empirical evaluation \cite{sun2025omega, EscapeBench2025, schapiro2026twistbench}, as well as its theoretical foundations \cite{Schapiro2025TransformationalTheory}, both of which should inform future task design for LLMs.

\subsection{Limitations} \label{sec:limitations}
Before concluding, several limitations of our work are worth discussing. First, some benchmarks in our suite have $n < 25$ models with full coverage of the capability stack due to limited coverage on already-evaluated models (see~\Cref{tab:correlations}). Validity and specificity should be interpreted cautiously here, and future work can aim to expand benchmark coverage to ensure greater statistical reliability. Although we study a large set of LLMs, this pool consists exclusively of instruction-tuned LLMs, and our findings may not generalize to base models, as post-training often reduces diversity~\cite{yue2025does}. Lastly, even though semantic distance tests offer convenient and controllable scoring, there are inherent limitations to how much variance in creative achievement minimal assessments can capture, due to the complex, long-horizon, and open-ended nature of real-world tasks. \cite{kapoor2026openworld} has recently called for \emph{open-world evaluations}: ``long-horizon, messy, real-world tasks assessed through small-sample qualitative analysis rather than benchmark-scale automation.'' Accordingly, creative achievement may be best measured through a combination of minimal tests and open-ended, long-horizon assessments.
\section{Conclusion} \label{sec:conclusion}
In this work, we conducted the first systematic study of how well semantic distance tests predict creative achievement in LLMs. We find that no test significantly predicts output diversity, suggesting that point and distributional creativity may dissociate in LLMs, a pattern we call \textbf{\emph{jagged creativity}}. The answer to the question in the title of this work depends on which test is being used and which construct it aims to predict: while the DAT \emph{does} predict creative writing abilities in LLMs, the RAT is effectively a general capability proxy, and no single test predicts all constructs well. Overall, future work should exercise caution when applying such tests to make claims about the creativity of LLMs.

\section*{Acknowledgements}
The authors thank Roger E. Beaty, Babak Hemmatian, and Lav R. Varshney for helpful feedback in preparing the manuscript. 

A.G: This material is based upon work supported by the National Science Foundation Graduate Research Fellowship Program under Grant No. DGE 21-46756. Any opinions, findings, and conclusions or
recommendations expressed in this material are those of the author(s) and do not necessarily reflect
the views of the National Science Foundation.

S.S: This material is based upon work supported by the Lambda Grant Research Program.

\bibliographystyle{plain}
\bibliography{bib/main}

\newpage
\appendix
\onecolumn
\newpage
\section{Derivation of the Validity-Specificity Frontier}
\label{app:spec-bound}

The theoretical ceilings plotted in Figure~\ref{fig:headline} (benchmark validity-specificity frontier) are built on a covariance matrix bound that we prove here.

\noindent\textbf{Theorem~\ref{thm:spec-ceiling} (Validity-Specificity Frontier, restated).}\;\emph{Let $g = (Z_1, \dots, Z_k)$ be a vector of general capability proxies, let $\hat{Y}_g$ denote the ordinary least squares prediction of $Y$ from $g$, and define $R = \mathrm{corr}(Y, \hat{Y}_g)$. For any test $X$ with validity $v = \mathrm{corr}(X, Y)$, the semi-partial correlation $r(X, Y \mid g) := r(X,\, Y - \hat{Y}_g)$ satisfies}
\begin{equation*}
    |r(X, Y \mid g)| \;\leq\; v\sqrt{1 - R^2} \,+\, |R|\,\sqrt{1 - v^2}.
\end{equation*}

\begin{proof}
Without loss of generality, we standardize $X$ and $Y$ to have unit variance, and let $\tilde{Z} := \hat{Y}_g / R$ so that $\sigma(\tilde{Z}) = 1$ as well. Since correlation is scale-invariant, $\mathrm{corr}(X, \tilde{Z}) = \mathrm{corr}(X, \hat{Y}_g) =: a$ and $\mathrm{corr}(Y, \tilde{Z}) = R$. The covariance matrix of the unit-variance triple $(X, Y, \tilde{Z})$,
\[
\Sigma \;=\;
\begin{pmatrix}
1 & v & a \\
v & 1 & R \\
a & R & 1
\end{pmatrix},
\]
must be positive semi-definite. Expanding the determinant gives
\begin{align}
    \det(\Sigma) &= 1 + 2vRa - a^2 - v^2 - R^2 \;\ge\; 0\\
    & \implies \underbrace{1}_{\alpha} a^2 \;-\; \underbrace{2vR}_{\beta} a \;+\; \underbrace{(R^2 + v^2 - 1)}_{\tau} \;\leq\; 0. \label{eq:quadratic}
\end{align}
Applying the quadratic formula to Equation~\eqref{eq:quadratic} with coefficients $\alpha, \beta, \tau$ yields
\begin{equation}
    a \;\in\; \bigl[\,vR - \delta,\;\; vR + \delta\,\bigr],
    \qquad \delta \;:=\; \sqrt{(1 - R^2)(1 - v^2)}. \label{eq:a-range}
\end{equation}
Finally, applying this into the semi-partial correlation: 
\begin{align}
    r(X, Y \mid g) &= r(X, Y - \hat{Y}_g) \\
    &= \frac{\textrm{Cov}(X, Y - \hat{Y}_g)}{\sigma(X)\,\sigma(Y - \hat{Y}_g)} \\
    &= \frac{\textrm{Cov}(X, Y) - \textrm{Cov}(X, \hat{Y}_g)}{\sigma(Y - \hat{Y}_g)} \quad \quad \quad\quad\quad\quad\quad\text{($X$ has unit variance)} \\
    &= \frac{v - a\,\sigma(\hat{Y}_g)}{\sigma(Y - \hat{Y}_g)}. \label{eq:semipartial-pre}
\end{align}
The last step substitutes $v := \textrm{Cov}(X, Y)$ and applies $\textrm{Cov}(X, \hat{Y}_g) = \textrm{corr}(X, \hat{Y}_g) ~\sigma(X)~ \sigma(\hat{Y}_g) = a \cdot 1 \cdot \sigma(\hat{Y}_g) = a\,\sigma(\hat{Y}_g)$, where the second equality uses $\textrm{corr}(X, \hat{Y}_g) = a$ (by definition) and $\sigma(X) = 1$ (since we standardized $X$). 

From here, it remains to address $\sigma(\hat{Y}_g)$ and $\sigma(Y - \hat{Y}_g)$. The former follows directly from our definition of $R := \mathrm{corr}(Y, \hat{Y}_g)$, along with the fact that least-squares residuals are orthogonal to their predictor (i.e, $\textrm{Cov}(Y - \hat{Y}_g,\, \hat{Y}_g) = 0$):
\begin{align}
R &= \mathrm{corr}(Y, \hat{Y}_g)  \\
  &= \frac{\textrm{Cov}(Y, \hat{Y}_g)}{\sigma(Y)\,\sigma(\hat{Y}_g)} \\
  &= \frac{\textrm{Cov}(\hat{Y}_g,\, \hat{Y}_g) + \textrm{Cov}(Y - \hat{Y}_g,\, \hat{Y}_g)}{\sigma(Y)\,\sigma(\hat{Y}_g)} \\
  &= \frac{\textrm{Var}(\hat{Y}_g)}{\sigma(Y)\,\sigma(\hat{Y}_g)}
   \quad \quad (\textrm{Cov}(Y - \hat{Y}_g,\, \hat{Y}_g) = 0 \text{ by least-squares orthogonality}) \\
  &= \frac{\sigma(\hat{Y}_g)}{\sigma(Y)} = \sigma(\hat{Y}_g)
   \quad \quad (\textrm{Var}(\hat{Y}_g) = \sigma^2(\hat{Y}_g) \text{ and } \sigma(Y) = 1).
\end{align}
Hence $\sigma(\hat{Y}_g) = R$ and $\sigma^2(\hat{Y}_g) = R^2$. Then, expanding $\sigma(Y - \hat{Y}_g)$ yields:
\begin{align}
    \sigma(Y - \hat{Y}_g) &= \sqrt{\textrm{Var}(Y - \hat{Y}_g)} \\
    &= \sqrt{\textrm{Var}(Y) + \textrm{Var}(\hat{Y}_g) - 2 \textrm{Cov}(Y, \hat{Y}_g)} \\
    &= \sqrt{1 + R^2 - 2\cdot \textrm{corr}(Y, \hat{Y}_g) \cdot \sigma(\hat{Y}_g) \cdot \sigma(Y)} \quad \quad (\textrm{Since $\tfrac{\textrm{Cov}(Y, \hat{Y}_g)}{\sigma(Y) \sigma(\hat{Y}_g)} = \textrm{corr}(Y, \hat{Y}_g) =: R$})\\
    &= \sqrt{1 + R^2 - 2R^2} \\
    &= \sqrt{1 - R^2}
\end{align}

Substituting back into \eqref{eq:semipartial-pre} and combining with the bounds on $a$ from Equation~\eqref{eq:a-range}, we obtain
\begin{align}
    r(X, Y | g)&= \frac{v - aR}{\sqrt{1 - R^2}} \\
    &\in \left[\,\frac{v - (vR + \delta)R}{\sqrt{1 - R^2}},\;\; \frac{v - (vR - \delta)R}{\sqrt{1 - R^2}}\,\right] \quad\quad\quad\quad\quad\text{(from Eq.~\eqref{eq:a-range})} \\
    &= \bigl[\,v\sqrt{1 - R^2} - |R|\sqrt{1 - v^2},\;\; v\sqrt{1 - R^2} + |R|\sqrt{1 - v^2}\,\bigr],
\end{align}
which gives the bound in Equation~\eqref{eq:spec-ceiling}.
\end{proof}

\paragraph{Restated Remark on~\Cref{thm:spec-ceiling}.} In particular, a perfectly valid test ($v = 1$) has specificity bounded by $\sqrt{1 - R^2}$; a benchmark with $R$ near $\pm 1$ therefore leaves almost no specificity headroom for \emph{any} test, regardless of construction. This is why Arena CW ($R = 0.87$ on $g = \text{MMLU-Pro}$) caps specificity at $\approx 0.49$ for a perfect-validity test, while NoveltyBench Utility ($R \approx -0.32$) admits a much wider ceiling (Figure~\ref{fig:headline}, bottom row).

\clearpage
\section{Validity and Specificity for Every Test--Benchmark Pair} \label{app:correlations}
\Cref{tab:correlations} reports the validity ($r$) and specificity ($r|g$, with $g$ = MMLU-Pro) of every semantic distance test on every benchmark, which \Cref{fig:headline} summarizes by construct and by benchmark. The Overall block uses each test's mean $z$-score across the three embeddings and is the main result referenced in \Cref{sec:results}. The GloVe, FastText, and Sentence-BERT blocks below it report each embedding separately as a robustness check. Benchmarks are grouped by construct, and the number of models behind each validity and specificity estimate is reported under each benchmark.

\begin{table*}[!h]
\centering
\setlength{\tabcolsep}{2.5pt}
\resizebox{\textwidth}{!}{%
\small
\begin{tabular}{@{}l c >{\columncolor{gray!10}}c c >{\columncolor{gray!10}}c c >{\columncolor{gray!10}}c c >{\columncolor{gray!10}}c c >{\columncolor{gray!10}}c c >{\columncolor{gray!10}}c c >{\columncolor{gray!10}}c@{}}
\toprule
 & \multicolumn{6}{c}{\textit{Creative Writing}}
 & \multicolumn{4}{c}{\textit{Output Diversity}}
 & \multicolumn{2}{c}{\textit{\begin{tabular}[b]{@{}c@{}}Combinatorial \\[-1pt] Creativity\end{tabular}}}
 & \multicolumn{2}{c}{\textit{Scientific Ideation}} \\
 \cmidrule(lr){2-7} \cmidrule(lr){8-11} \cmidrule(lr){12-13} \cmidrule(lr){14-15}
 & \multicolumn{2}{c}{\textbf{Arena CW $\uparrow$}}
 & \multicolumn{2}{c}{\textbf{EQ-Bench CW $\uparrow$}}
 & \multicolumn{2}{c}{\textbf{Mazur CW $\uparrow$}}
 & \multicolumn{2}{c}{\textbf{Hivemind Div. $\uparrow$}}
 & \multicolumn{2}{c}{\textbf{NovBench Util. $\uparrow$}}
 & \multicolumn{2}{c}{\textbf{Kombine $\uparrow$}}
 & \multicolumn{2}{c}{\textbf{LiveIdeaBench $\uparrow$}} \\
 & \multicolumn{2}{c}{\footnotesize $n{=}54/40$} & \multicolumn{2}{c}{\footnotesize $n{=}35/27$} & \multicolumn{2}{c}{\footnotesize $n{=}21/18$} & \multicolumn{2}{c}{\footnotesize $n{=}25/21$} & \multicolumn{2}{c}{\footnotesize $n{=}27/24$} & \multicolumn{2}{c}{\footnotesize $n{=}20/18$} & \multicolumn{2}{c}{\footnotesize $n{=}28/22$} \\
 \cmidrule(lr){2-3} \cmidrule(lr){4-5} \cmidrule(lr){6-7} \cmidrule(lr){8-9} \cmidrule(lr){10-11} \cmidrule(lr){12-13} \cmidrule(lr){14-15}
 & Valid. & Spec. & Valid. & Spec. & Valid. & Spec. & Valid. & Spec. & Valid. & Spec. & Valid. & Spec. & Valid. & Spec. \\
\midrule
\multicolumn{15}{@{}l}{\textbf{Overall (mean $z$-score across 3 embeddings)}} \\
DAT & \good{$+.49^{***}$} & \good{$+.39^{*}$} & \good{$+.76^{***}$} & \good{$+.53^{**}$} & \good{$+.61^{**}$} & \good{$+.68^{**}$} & $+.30$ & $-.01$ & $+.06$ & $+.08$ & \good{$+.69^{***}$} & $+.32$ & $+.23$ & $+.20$ \\
CDAT & $-.00$ & $+.01$ & $+.04$ & $+.01$ & $+.29$ & $+.39$ & $+.25$ & $+.31$ & $+.25$ & $+.25$ & $-.06$ & $-.26$ & $+.05$ & $-.04$ \\
\quad CDAT-N & $-.12$ & $-.15$ & $-.04$ & $+.05$ & $+.14$ & $+.24$ & $+.20$ & $+.34$ & $+.32$ & $+.31$ & $-.21$ & $-.19$ & $-.03$ & $-.15$ \\
\quad CDAT-A & \good{$+.50^{***}$} & \good{$+.33^{*}$} & \good{$+.42^{*}$} & $+.13$ & $+.21$ & $+.04$ & $-.35$ & $-.37$ & $-.32$ & $-.31$ & \good{$+.58^{**}$} & $+.35$ & $+.15$ & $+.22$ \\
PACE & \good{$+.71^{***}$} & \good{$+.41^{*}$} & \good{$+.71^{***}$} & $+.30$ & \good{$+.76^{***}$} & \good{$+.52^{*}$} & $-.05$ & $+.07$ & $+.10$ & $+.17$ & \good{$+.70^{***}$} & $+.22$ & \good{$+.44^{*}$} & $+.28$ \\
RAT & \good{$+.62^{***}$} & $+.15$ & \good{$+.44^{*}$} & $+.08$ & \good{$+.51^{*}$} & $+.28$ & $-.24$ & $-.12$ & $-.02$ & $-.01$ & \good{$+.58^{**}$} & $+.14$ & $+.23$ & $+.07$ \\
\midrule
\multicolumn{15}{@{}l}{\textit{GloVe 840B}} \\
DAT & \good{$+.40^{**}$} & $+.29$ & \good{$+.59^{***}$} & $+.31$ & \good{$+.49^{*}$} & \good{$+.55^{*}$} & $+.20$ & $+.03$ & $-.05$ & $-.05$ & \good{$+.66^{**}$} & $+.35$ & $+.26$ & $+.16$ \\
CDAT & $+.03$ & $+.02$ & $+.13$ & $+.08$ & $+.38$ & $+.47$ & $+.33$ & $+.35$ & $+.24$ & $+.25$ & $+.17$ & $-.08$ & $+.09$ & $+.02$ \\
\quad CDAT-N & $-.10$ & $-.13$ & $+.02$ & $+.11$ & $+.23$ & $+.31$ & $+.23$ & $+.36$ & $+.31$ & $+.31$ & $-.01$ & $-.06$ & $+.00$ & $-.11$ \\
\quad CDAT-A & \good{$+.50^{***}$} & \good{$+.33^{*}$} & \good{$+.40^{*}$} & $+.08$ & $+.13$ & $-.04$ & $-.36$ & $-.38$ & $-.31$ & $-.30$ & \good{$+.47^{*}$} & $+.28$ & $+.12$ & $+.18$ \\
PACE & \good{$+.72^{***}$} & $+.31$ & \good{$+.70^{***}$} & $+.26$ & \good{$+.71^{***}$} & $+.46$ & $-.13$ & $+.03$ & $+.06$ & $+.14$ & \good{$+.74^{***}$} & $+.28$ & \good{$+.43^{*}$} & $+.28$ \\
\midrule
\multicolumn{15}{@{}l}{\textit{FastText crawl-300d-2M}} \\
DAT & \good{$+.35^{**}$} & \good{$+.36^{*}$} & \good{$+.61^{***}$} & \good{$+.47^{*}$} & \good{$+.60^{**}$} & \good{$+.64^{**}$} & $+.30$ & $-.16$ & $+.14$ & $+.17$ & \good{$+.59^{**}$} & $+.36$ & $+.01$ & $-.03$ \\
CDAT & $+.00$ & $+.05$ & $+.02$ & $-.02$ & $+.16$ & $+.29$ & $+.19$ & $+.23$ & $+.21$ & $+.21$ & $-.26$ & $-.43$ & $+.04$ & $-.05$ \\
\quad CDAT-N & $-.08$ & $-.06$ & $-.03$ & $+.01$ & $+.06$ & $+.19$ & $+.18$ & $+.26$ & $+.26$ & $+.26$ & $-.36$ & $-.35$ & $-.03$ & $-.16$ \\
\quad CDAT-A & \good{$+.50^{***}$} & $+.30$ & \good{$+.44^{**}$} & $+.17$ & $+.27$ & $+.09$ & $-.36$ & $-.34$ & $-.30$ & $-.29$ & \good{$+.67^{**}$} & $+.41$ & $+.18$ & $+.25$ \\
PACE & \good{$+.72^{***}$} & \good{$+.44^{**}$} & \good{$+.71^{***}$} & $+.33$ & \good{$+.71^{***}$} & \good{$+.53^{*}$} & $-.13$ & $+.02$ & $+.09$ & $+.15$ & \good{$+.60^{**}$} & $+.12$ & $+.39$ & $+.21$ \\
\midrule
\multicolumn{15}{@{}l}{\textit{Sentence-BERT all-mpnet-base-v2}} \\
DAT & \good{$+.46^{***}$} & \good{$+.32^{*}$} & \good{$+.67^{***}$} & \good{$+.52^{**}$} & $+.42$ & \good{$+.50^{*}$} & $+.21$ & $+.12$ & $+.00$ & $+.02$ & \good{$+.57^{**}$} & $+.14$ & $+.28$ & $+.38$ \\
CDAT & $-.04$ & $-.05$ & $-.01$ & $-.02$ & $+.30$ & $+.36$ & $+.22$ & $+.34$ & $+.28$ & $+.28$ & $-.14$ & $-.28$ & $+.00$ & $-.09$ \\
\quad CDAT-N & $-.18$ & $-.24$ & $-.11$ & $+.03$ & $+.12$ & $+.20$ & $+.19$ & $+.38$ & $+.36$ & $+.35$ & $-.28$ & $-.17$ & $-.06$ & $-.16$ \\
\quad CDAT-A & \good{$+.51^{***}$} & \good{$+.36^{*}$} & \good{$+.43^{*}$} & $+.15$ & $+.23$ & $+.08$ & $-.32$ & $-.38$ & $-.35$ & $-.33$ & \good{$+.59^{**}$} & $+.34$ & $+.15$ & $+.23$ \\
PACE & \good{$+.60^{***}$} & \good{$+.42^{**}$} & \good{$+.65^{***}$} & $+.30$ & \good{$+.76^{***}$} & \good{$+.51^{*}$} & $+.18$ & $+.19$ & $+.15$ & $+.20$ & \good{$+.70^{***}$} & $+.22$ & \good{$+.44^{*}$} & $+.33$ \\
\bottomrule
\end{tabular}
}
\vspace{2pt}
\caption{\textbf{Validity and specificity for every test--benchmark pair.} The \textbf{Overall} block (top) averages across the three embeddings and is the main result referenced in the text. For robustness, we also report per-embedding blocks; the RAT is scored by accuracy rather than embedding distance, so it appears only in the Overall block. Bold = significant at $p<.05$, while stars denote significance levels ($^*p<.05$, $^{**}p<.01$, $^{***}p<.001$). CDAT is novelty averaged only over the temperatures at which a model passes the appropriateness gate (Welch $t$-test against random nouns, BH-FDR across models, adjusted $p < .001$); CDAT-N is novelty over all temperatures, and CDAT-A is raw appropriateness. Lastly, sample sizes $n = \text{val}/\text{spec}$ are reported under each header, and specificity has reduced sample sizes because not every model has an MMLU-Pro score.}
\label{tab:correlations}
\vspace{-10pt}
\end{table*}

\clearpage
\section{Validity and Specificity by Sampling Temperature} \label{app:per-temperature}
The DAT and CDAT are administered at $T \in \{1.0, 1.5, 2.0\}$, following prior work \cite{wang2025large, nakajima2026beyond}, while PACE is administered at $T = 0$ following \cite{pace}. \Cref{tab:per-temperature} reports the validity and specificity of the DAT and CDAT at each temperature separately. Because per-temperature scores are computed under each test's primary embedding (GloVe for the DAT and Sentence-BERT for the CDAT), values differ slightly from the Overall block of \Cref{tab:correlations}. Model rankings are highly stable across temperatures (DAT $r \in [0.80, 0.93]$; CDAT $r \in [0.82, 0.89]$; all $p < 10^{-12}$). For reference, the benchmarks are collected at the following temperatures: NoveltyBench and Hivemind at $T = 1.0$, Kombine at $T = 0.9$, EQ-Bench CW and Arena CW at $T = 0.7$, LiveIdeaBench at the API default, and MMLU-Pro at $T = 0$, while Mazur CW does not document its temperature.

\begin{table*}[!h]
\centering
\setlength{\tabcolsep}{2.5pt}
\resizebox{\textwidth}{!}{%
\small
\begin{tabular}{@{}l c >{\columncolor{gray!10}}c c >{\columncolor{gray!10}}c c >{\columncolor{gray!10}}c c >{\columncolor{gray!10}}c c >{\columncolor{gray!10}}c c >{\columncolor{gray!10}}c c >{\columncolor{gray!10}}c@{}}
\toprule
 & \multicolumn{6}{c}{\textit{Creative Writing}}
 & \multicolumn{4}{c}{\textit{Output Diversity}}
 & \multicolumn{2}{c}{\textit{\begin{tabular}[b]{@{}c@{}}Combinatorial \\[-1pt] Creativity\end{tabular}}}
 & \multicolumn{2}{c}{\textit{Scientific Ideation}} \\
 \cmidrule(lr){2-7} \cmidrule(lr){8-11} \cmidrule(lr){12-13} \cmidrule(lr){14-15}
 & \multicolumn{2}{c}{\textbf{Arena CW $\uparrow$}}
 & \multicolumn{2}{c}{\textbf{EQ-Bench CW $\uparrow$}}
 & \multicolumn{2}{c}{\textbf{Mazur CW $\uparrow$}}
 & \multicolumn{2}{c}{\textbf{Hivemind Div. $\uparrow$}}
 & \multicolumn{2}{c}{\textbf{NovBench Util. $\uparrow$}}
 & \multicolumn{2}{c}{\textbf{Kombine $\uparrow$}}
 & \multicolumn{2}{c}{\textbf{LiveIdeaBench $\uparrow$}} \\
 & \multicolumn{2}{c}{\footnotesize $n{=}54/40$} & \multicolumn{2}{c}{\footnotesize $n{=}35/27$} & \multicolumn{2}{c}{\footnotesize $n{=}21/18$} & \multicolumn{2}{c}{\footnotesize $n{=}25/21$} & \multicolumn{2}{c}{\footnotesize $n{=}27/24$} & \multicolumn{2}{c}{\footnotesize $n{=}20/18$} & \multicolumn{2}{c}{\footnotesize $n{=}26/21$} \\
 \cmidrule(lr){2-3} \cmidrule(lr){4-5} \cmidrule(lr){6-7} \cmidrule(lr){8-9} \cmidrule(lr){10-11} \cmidrule(lr){12-13} \cmidrule(lr){14-15}
 & Valid. & Spec. & Valid. & Spec. & Valid. & Spec. & Valid. & Spec. & Valid. & Spec. & Valid. & Spec. & Valid. & Spec. \\
\midrule
\multicolumn{15}{@{}l}{\textbf{DAT} \textit{(GloVe 840B)}} \\
\quad $T = 1.0$ & \good{$+.40^{**}$} & \good{$+.34^{*}$} & \good{$+.59^{***}$} & $+.32$ & \good{$+.50^{*}$} & \good{$+.52^{*}$} & $+.28$ & $+.08$ & $-.13$ & $-.13$ & \good{$+.72^{***}$} & $+.32$ & $+.29$ & $+.23$ \\
\quad $T = 1.5$ & \good{$+.47^{***}$} & $+.31$ & \good{$+.64^{***}$} & $+.25$ & \good{$+.52^{*}$} & \good{$+.61^{*}$} & $+.09$ & $-.07$ & $-.05$ & $-.10$ & \good{$+.63^{**}$} & $+.35$ & $+.19$ & $+.26$ \\
\quad $T = 2.0$ & \good{$+.44^{**}$} & $+.25$ & \good{$+.66^{***}$} & $+.29$ & $+.39$ & \good{$+.57^{*}$} & $-.09$ & $-.32$ & $-.04$ & $-.11$ & $+.41$ & $+.41$ & $+.15$ & $+.19$ \\
\midrule
\multicolumn{15}{@{}l}{\textbf{CDAT} \textit{(Sentence-BERT all-mpnet-base-v2)}} \\
\quad $T = 1.0$ & $+.01$ & $+.02$ & $+.14$ & $+.09$ & $+.43$ & $+.45$ & $+.39$ & $+.41$ & $+.19$ & $+.20$ & $+.19$ & $-.15$ & $-.03$ & $-.04$ \\
\quad $T = 1.5$ & $-.10$ & $-.10$ & $-.05$ & $-.09$ & $+.11$ & $+.13$ & $+.27$ & $+.35$ & \good{$+.41^{*}$} & $+.36$ & $-.26$ & $-.26$ & $+.05$ & $+.07$ \\
\quad $T = 2.0$ & $+.11$ & $+.16$ & $-.06$ & $-.18$ & $+.25$ & $+.39$ & $-.16$ & $-.10$ & $+.15$ & $+.08$ & $-.44$ & $-.29$ & $+.08$ & $+.02$ \\
\bottomrule
\end{tabular}
}
\vspace{2pt}
\caption{\textbf{Validity and specificity of the DAT and CDAT at each sampling temperature.}}
\label{tab:per-temperature}
\end{table*}

\clearpage
\section{Per-Model Test Scores} \label{app:per-model-scores}
In Table~\ref{tab:per-model-scores}, we report the mean~$\pm$~SEM scores for the DAT, CDAT, CDAT-N, CDAT-A, and PACE per model. DAT is scored under GloVe across all valid trials at $T \in \{1.0, 1.5, 2.0\}$ (at 120 total trials per model). CDAT, CDAT-N, and CDAT-A are scored under Sentence-BERT \texttt{all-mpnet-base-v2}, where CDAT is novelty restricted to the temperatures passing the false discovery rate $\alpha\!=\!.001$ appropriateness gate (``---'' if no temperature passed), and CDAT-N and CDAT-A are the ungated novelty and appropriateness scores aggregated over all 50 cues at all three temperatures. PACE is scored under FastText across all valid chains (up to 150 chains per model: 50 seeds $\times$ 3 chains).

As a sanity check, our per-model mean~$\pm$~SEM values fall within the score ranges reported in the original papers. PACE across our 54 models spans $\sim$0.65--0.76, overlapping the $\sim$0.69--0.83 range reported by \cite{pace} across their 30-model set. Gated CDAT novelty across our models spans $\sim$62--75, within the $\sim$55--75 band shown in Figure~3 of \cite{nakajima2026beyond}. Our evaluation extends both sets to a larger number of models.

\scriptsize
\setlength{\tabcolsep}{3pt}
\begin{longtable}{@{}lrrrrrr@{}}
\caption{Mean~$\pm$~SEM of DAT, CDAT, CDAT-N, CDAT-A, PACE, and RAT per model, grouped by provider. ``---'' indicates cells where no valid scores were collected (no temperature passed the appropriateness gate) on the CDAT, or responses were otherwise invalid.}
\label{tab:per-model-scores} \\
\toprule
Model & DAT & CDAT & CDAT-N  & CDAT-A  & PACE & RAT \\
\midrule
\endfirsthead
\multicolumn{7}{@{}l}{\textit{(continued from previous page)}} \\
\toprule
Model & DAT  & CDAT  & CDAT-N  & CDAT-A  & PACE & RAT \\
\midrule
\endhead
\midrule
\multicolumn{7}{r@{}}{\textit{(continued on next page)}} \\
\endfoot
\bottomrule
\endlastfoot
\multicolumn{7}{@{}l}{\raisebox{-0.25ex}{\includegraphics[height=1.8ex]{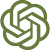}}\ \textit{OpenAI}} \\
\texttt{gpt-3-5-turbo} & 78.94$\pm$0.34 & 72.66$\pm$0.39 & 72.66$\pm$0.39 & 132.92$\pm$0.63 & 0.715$\pm$0.004 & 80.0 \\
\texttt{gpt-4-1} & 86.29$\pm$0.27 & 70.10$\pm$0.42 & 70.10$\pm$0.42 & 140.86$\pm$0.55 & 0.744$\pm$0.002 & 96.7 \\
\texttt{gpt-4-1-mini} & 81.60$\pm$0.49 & 67.75$\pm$0.58 & 67.75$\pm$0.58 & 143.78$\pm$0.66 & 0.730$\pm$0.003 & 80.0 \\
\texttt{gpt-4-1-nano} & 83.73$\pm$0.41 & 70.27$\pm$0.77 & 71.28$\pm$0.72 & 137.84$\pm$0.94 & 0.707$\pm$0.004 & 46.7 \\
\texttt{gpt-4-turbo} & 87.90$\pm$0.35 & 66.92$\pm$0.57 & 66.92$\pm$0.57 & 144.02$\pm$0.67 & 0.732$\pm$0.003 & 93.3 \\
\texttt{gpt-4o} & 84.47$\pm$0.43 & 66.26$\pm$0.69 & 66.26$\pm$0.69 & 144.54$\pm$0.78 & 0.729$\pm$0.002 & 93.3 \\
\texttt{gpt-4o-mini} & 80.90$\pm$0.46 & 71.52$\pm$0.46 & 71.52$\pm$0.46 & 138.41$\pm$0.69 & 0.707$\pm$0.003 & 50.0 \\
\texttt{gpt-5} & 89.65$\pm$0.18 & 69.76$\pm$0.46 & 69.76$\pm$0.46 & 142.26$\pm$0.55 & 0.747$\pm$0.002 & 93.3 \\
\texttt{gpt-5-4} & 91.72$\pm$0.19 & 68.63$\pm$0.49 & 68.63$\pm$0.49 & 143.70$\pm$0.54 & 0.727$\pm$0.003 & 93.3 \\
\texttt{gpt-5-4-mini} & 84.06$\pm$0.32 & 65.47$\pm$0.52 & 65.47$\pm$0.52 & 145.88$\pm$0.59 & 0.734$\pm$0.002 & 73.3 \\
\texttt{gpt-5-4-nano} & 83.77$\pm$0.29 & 63.15$\pm$0.58 & 63.15$\pm$0.58 & 147.61$\pm$0.64 & 0.678$\pm$0.004 & 33.3 \\
\texttt{gpt-5-mini} & 82.92$\pm$0.36 & 68.02$\pm$0.49 & 68.02$\pm$0.49 & 143.90$\pm$0.57 & 0.741$\pm$0.003 & 96.7 \\
\texttt{gpt-5-nano} & 80.66$\pm$0.32 & 63.11$\pm$0.59 & 63.11$\pm$0.59 & 147.95$\pm$0.63 & 0.712$\pm$0.003 & 93.3 \\
\texttt{o3} & 89.46$\pm$0.24 & 70.06$\pm$0.44 & 70.06$\pm$0.44 & 142.25$\pm$0.54 & 0.748$\pm$0.004 & 93.3 \\
\texttt{o3-mini} & 76.37$\pm$0.42 & 65.67$\pm$0.61 & 65.67$\pm$0.61 & 144.17$\pm$0.71 & 0.715$\pm$0.003 & 70.0 \\
\texttt{o4-mini} & 84.44$\pm$0.36 & 68.77$\pm$0.46 & 68.77$\pm$0.46 & 142.61$\pm$0.58 & 0.732$\pm$0.003 & 93.3 \\
\midrule
\multicolumn{7}{@{}l}{\raisebox{-0.25ex}{\includegraphics[height=1.8ex]{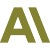}}\ \textit{Anthropic}} \\
\texttt{claude-3-5-haiku} & 87.36$\pm$0.22 & 67.95$\pm$0.43 & 67.95$\pm$0.43 & 141.52$\pm$0.52 & 0.726$\pm$0.002 & 63.3 \\
\texttt{claude-3-haiku} & 78.87$\pm$0.23 & 62.53$\pm$0.58 & 62.53$\pm$0.58 & 146.22$\pm$0.65 & 0.697$\pm$0.005 & 70.0 \\
\texttt{claude-haiku-4-5} & 85.74$\pm$0.26 & 67.80$\pm$0.51 & 67.80$\pm$0.51 & 143.71$\pm$0.65 & 0.667$\pm$0.010 & 83.3 \\
\texttt{claude-opus-4-5} & 89.26$\pm$0.21 & 69.90$\pm$0.38 & 69.90$\pm$0.38 & 143.03$\pm$0.46 & 0.742$\pm$0.003 & 86.7 \\
\texttt{claude-opus-4-6} & 91.24$\pm$0.16 & 66.84$\pm$0.54 & 66.84$\pm$0.54 & 146.49$\pm$0.59 & 0.750$\pm$0.002 & 86.7 \\
\texttt{claude-sonnet-4} & 86.69$\pm$0.16 & 69.78$\pm$0.37 & 69.78$\pm$0.37 & 141.69$\pm$0.53 & 0.739$\pm$0.004 & 96.7 \\
\texttt{claude-sonnet-4-5} & 86.67$\pm$0.17 & 66.43$\pm$0.48 & 66.43$\pm$0.48 & 146.01$\pm$0.56 & 0.756$\pm$0.002 & 90.0 \\
\texttt{claude-sonnet-4-6} & 89.63$\pm$0.16 & 69.80$\pm$0.32 & 69.80$\pm$0.32 & 143.38$\pm$0.43 & 0.755$\pm$0.002 & 86.7 \\
\midrule
\multicolumn{7}{@{}l}{\raisebox{-0.25ex}{\includegraphics[height=1.8ex]{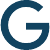}}\ \textit{Google}} \\
\texttt{gemini-2-0-flash-001} & 82.32$\pm$0.41 & 70.32$\pm$0.41 & 70.32$\pm$0.41 & 139.44$\pm$0.55 & 0.730$\pm$0.003 & 86.7 \\
\texttt{gemini-2-5-flash} & 76.68$\pm$0.54 & 68.54$\pm$0.49 & 68.54$\pm$0.49 & 143.39$\pm$0.58 & 0.742$\pm$0.002 & 76.7 \\
\texttt{gemini-2-5-pro} & 89.69$\pm$0.25 & 71.13$\pm$0.34 & 71.13$\pm$0.34 & 139.02$\pm$0.49 & 0.761$\pm$0.002 & 93.3 \\
\texttt{gemma-2-27b-it} & 81.87$\pm$0.44 & 70.92$\pm$0.61 & 70.92$\pm$0.61 & 138.29$\pm$0.85 & 0.694$\pm$0.008 & 50.0 \\
\texttt{gemma-2-9b-it} & 80.40$\pm$0.53 & 71.96$\pm$0.48 & 72.77$\pm$0.51 & 133.71$\pm$0.73 & 0.728$\pm$0.003 & --- \\
\texttt{gemma-3-27b-it} & 86.49$\pm$0.25 & 72.05$\pm$0.33 & 72.05$\pm$0.33 & 137.58$\pm$0.51 & 0.728$\pm$0.005 & 73.3 \\
\midrule
\multicolumn{7}{@{}l}{\raisebox{-0.25ex}{\includegraphics[height=1.8ex]{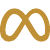}}\ \textit{Meta}} \\
\texttt{llama-3-1-70b-instruct} & 87.71$\pm$0.30 & 68.19$\pm$0.80 & 68.19$\pm$0.80 & 141.56$\pm$0.98 & 0.713$\pm$0.004 & 83.3 \\
\texttt{llama-3-1-8b-instruct} & 88.48$\pm$0.48 & 69.82$\pm$0.88 & 72.99$\pm$0.78 & 134.84$\pm$1.09 & 0.701$\pm$0.006 & 50.0 \\
\texttt{llama-3-2-1b-instruct} & 81.20$\pm$0.25 & 58.92$\pm$1.33 & 58.92$\pm$1.33 & 146.46$\pm$1.67 & 0.586$\pm$0.017 & 0.0 \\
\texttt{llama-3-2-3b-instruct} & 84.11$\pm$0.14 & 68.31$\pm$0.50 & 68.31$\pm$0.50 & 139.37$\pm$0.67 & 0.711$\pm$0.005 & 33.3 \\
\texttt{llama-3-3-70b-instruct} & 88.05$\pm$0.37 & 68.22$\pm$0.70 & 72.25$\pm$0.71 & 135.47$\pm$0.94 & 0.718$\pm$0.004 & 80.0 \\
\texttt{llama-4-maverick} & 85.28$\pm$0.26 & 67.45$\pm$0.44 & 67.45$\pm$0.44 & 141.98$\pm$0.56 & 0.707$\pm$0.005 & 76.7 \\
\texttt{llama-4-scout} & 85.91$\pm$0.28 & 67.39$\pm$0.50 & 67.39$\pm$0.50 & 141.42$\pm$0.61 & 0.696$\pm$0.005 & 73.3 \\
\midrule
\multicolumn{7}{@{}l}{\raisebox{-0.25ex}{\includegraphics[height=1.8ex]{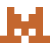}}\ \textit{Mistral}} \\
\texttt{mistral-7b-instruct-v0-1} & 81.20$\pm$0.01 & 69.16$\pm$0.54 & 69.16$\pm$0.54 & 135.43$\pm$0.56 & 0.613$\pm$0.011 & 16.7 \\
\texttt{mistral-large-2407} & 88.15$\pm$0.24 & 70.36$\pm$0.34 & 70.36$\pm$0.34 & 138.80$\pm$0.55 & 0.737$\pm$0.002 & 96.7 \\
\texttt{mistral-large-2411} & 81.91$\pm$0.40 & 64.54$\pm$0.49 & 64.54$\pm$0.49 & 146.01$\pm$0.54 & 0.722$\pm$0.003 & 73.3 \\
\texttt{mistral-nemo} & 84.19$\pm$0.77 & 69.45$\pm$0.97 & 71.09$\pm$0.95 & 137.00$\pm$1.28 & 0.709$\pm$0.005 & 43.3 \\
\texttt{mistral-small-24b-instruct-2501} & 84.40$\pm$0.59 & 73.10$\pm$0.78 & 73.56$\pm$0.76 & 132.24$\pm$0.94 & 0.717$\pm$0.003 & 70.0 \\
\midrule
\multicolumn{7}{@{}l}{\raisebox{-0.25ex}{\includegraphics[height=1.8ex]{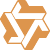}}\ \textit{Qwen}} \\
\texttt{qwen-2-5-72b-instruct} & 79.79$\pm$0.80 & 68.51$\pm$0.81 & 68.98$\pm$0.82 & 138.68$\pm$0.92 & 0.703$\pm$0.004 & 66.7 \\
\texttt{qwen3-14b} & 81.36$\pm$1.76 & --- & 73.89$\pm$1.20 & 131.84$\pm$1.70 & 0.606$\pm$0.017 & 46.7 \\
\texttt{qwen3-235b-a22b} & 84.95$\pm$0.45 & 68.63$\pm$0.62 & 68.63$\pm$0.62 & 142.18$\pm$0.78 & 0.725$\pm$0.003 & 90.0 \\
\texttt{qwen3-32b} & 85.18$\pm$1.57 & 72.56$\pm$2.30 & 74.54$\pm$1.34 & 135.62$\pm$1.72 & 0.658$\pm$0.019 & 70.0 \\
\texttt{qwen3-8b} & 83.31$\pm$0.35 & 68.53$\pm$0.45 & 68.53$\pm$0.45 & 141.87$\pm$0.60 & 0.694$\pm$0.004 & 40.0 \\
\texttt{qwq-32b} & 82.63$\pm$0.50 & --- & --- & --- & --- & --- \\
\midrule
\multicolumn{7}{@{}l}{\raisebox{-0.25ex}{\includegraphics[height=1.8ex]{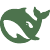}}\ \textit{DeepSeek}} \\
\texttt{deepseek-chat} & 81.12$\pm$0.40 & 67.50$\pm$0.53 & 67.50$\pm$0.53 & 143.11$\pm$0.71 & 0.729$\pm$0.003 & 80.0 \\
\texttt{deepseek-chat-v3-0324} & 84.80$\pm$0.43 & 70.63$\pm$0.44 & 70.63$\pm$0.44 & 137.63$\pm$0.66 & 0.730$\pm$0.004 & 73.3 \\
\texttt{deepseek-r1} & 83.80$\pm$0.39 & 69.28$\pm$0.45 & 69.28$\pm$0.45 & 141.38$\pm$0.65 & 0.720$\pm$0.003 & 60.0 \\
\midrule
\multicolumn{7}{@{}l}{\raisebox{-0.25ex}{\includegraphics[height=1.8ex]{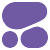}}\ \textit{Cohere}} \\
\texttt{command-a} & 82.28$\pm$0.33 & 63.22$\pm$0.64 & 63.22$\pm$0.64 & 147.80$\pm$0.66 & 0.714$\pm$0.003 & 73.3 \\
\texttt{command-r-plus-08-2024} & 87.69$\pm$0.45 & 69.44$\pm$0.49 & 69.44$\pm$0.49 & 138.70$\pm$0.64 & 0.722$\pm$0.003 & 83.3 \\
\midrule
\multicolumn{7}{@{}l}{\raisebox{-0.25ex}{\includegraphics[height=1.8ex]{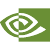}}\ \textit{NVIDIA}} \\
\texttt{llama-3-1-nemotron-70b-instruct} & 88.21$\pm$0.41 & 69.62$\pm$0.77 & 69.62$\pm$0.76 & 140.32$\pm$0.93 & 0.638$\pm$0.013 & --- \\
\midrule
\multicolumn{7}{@{}l}{\raisebox{-0.25ex}{\includegraphics[height=1.8ex]{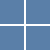}}\ \textit{Microsoft}} \\
\texttt{phi-4} & 82.43$\pm$0.81 & 69.74$\pm$0.83 & 73.54$\pm$0.78 & 132.55$\pm$1.06 & 0.680$\pm$0.006 & 63.3 \\
\end{longtable}
\normalsize

\clearpage
\section{Per-Model Benchmark Scores} \label{app:per-model-benchmarks}
Table~\ref{tab:per-model-benchmarks} reports each model's score on every external benchmark used in this study. The general capability proxy is MMLU-Pro accuracy~\cite{wang2024mmlupro}, the creative writing benchmarks are Arena CW (Elo), EQ-Bench CW v3 (Elo)~\cite{paech2023eqbench}, and Mazur CW~\cite{mazur2025writing}, the output diversity benchmarks are Hivemind diversity ($1 - $ intra-model cosine similarity)~\cite{hivemind} and NoveltyBench Utility (NovB.; cumulative-utility score)~\cite{zhang2025noveltybench}, the combinatorial creativity benchmark is the Kombine overall score~\cite{schapiro2026measuringcombinatorialcreativitylarge}, and the scientific ideation benchmark is the LiveIdeaBench~\cite{ruan2024liveideabench} \emph{idea score average} across five dimensions: originality, flexibility, feasibility, clarity, and fluency. Cells marked ``---'' are missing because the corresponding benchmark does not score that model---future work can expand the coverage of these benchmarks to improve the statistical robustness of observed correlations.

\paragraph{Mazur CW snapshot.} Mazur CW scores are transcribed from commit \texttt{80b7f17} (an absolute 0--10 mean-rubric leaderboard with 50 graded models). The leaderboard has since been regraded with a new ensemble of judges, eventually deprecating the absolute-rating system entirely in favor of a pairwise Thurstone ranking. We retained this snapshot rather than the latest version because it covers the largest set of models intersecting our evaluation pool ($n = 20$ overlap, vs.\ $\leq 10$ for any later snapshot), keeping per-cell sample sizes comparable to the other creative writing benchmarks.

\paragraph{MMLU-Pro source.} MMLU-Pro accuracies are taken from the \href{https://huggingface.co/spaces/TIGER-Lab/MMLU-Pro}{TIGER-Lab MMLU-Pro leaderboard} (CSV at \href{https://huggingface.co/datasets/TIGER-Lab/mmlu_pro_leaderboard_submission}{\texttt{TIGER-Lab/mmlu\_pro\_leaderboard\_submission}}), which is methodologically tied to the original MMLU-Pro paper~\cite{wang2024mmlupro}. Models without a leaderboard entry receive no MMLU-Pro value and are excluded from the specificity computation for that cell, which is why the per-benchmark specificity sample sizes are smaller than the validity sample sizes (Table~\ref{tab:correlations}).

\paragraph{Hivemind mean-similarity estimate.} \cite{hivemind} does not publish a mean intra-model similarity for each model. Instead, Table 6 in their paper reports, for each of 79 models, the percentage $p_b$ of response pairs whose pairwise cosine similarity falls into each of ten bins covering $[0, 1]$ in $0.1$-wide steps. We estimate the per-model mean similarity as the bin-midpoint-weighted sum, and report ``Hivemind diversity'' as $1 - \text{intra-sim}$.\footnote{This is exact when within-bin similarity values are concentrated at the bin midpoint and approximate (with worst-case error $\pm 0.05$) under uniform within-bin distributions; the rank order of models is, nonetheless, preserved unless within-bin distributions differ substantially across models}, where intra-sim is given by:
\begin{equation}
    \text{intra-sim} \;:=\; \tfrac{1}{100} \sum_{b=1}^{10} p_b \cdot m_b,
    \qquad m_b \in \{0.95, 0.85, \dots, 0.05\},
\end{equation}

\scriptsize
\setlength{\tabcolsep}{3pt}
\begin{longtable}{@{}lrrrrrrrr@{}}
\caption{\textbf{Per-model benchmark scores.} Columns: MMLU-Pro (accuracy), Arena CW (Elo), EQ-Bench CW v3 (Elo), Mazur CW, Hivemind diversity, NoveltyBench Utility, Kombine, LiveIdeaBench (5-dim Average). ``---'' indicates the corresponding benchmark does not score that model.}
\label{tab:per-model-benchmarks} \\
\toprule
Model & MMLU-Pro & Arena CW & EQ-B. CW & Mazur & Hive. & NovB. & Kombine & LiveIdea \\
\midrule
\endfirsthead
\multicolumn{9}{@{}l}{\textit{(continued from previous page)}} \\
\toprule
Model & MMLU-Pro & Arena CW & EQ-B. CW & Mazur & Hive. & NovB. & Kombine & LiveIdea \\
\midrule
\endhead
\midrule
\multicolumn{9}{r@{}}{\textit{(continued on next page)}} \\
\endfoot
\bottomrule
\endlastfoot
\multicolumn{9}{@{}l}{\raisebox{-0.25ex}{\includegraphics[height=1.8ex]{figures/neurips/logos/openai.pdf}}\ \textit{OpenAI}} \\
\texttt{gpt-3-5-turbo} & --- & 1187 & 519 & --- & --- & --- & --- & --- \\
\texttt{gpt-4-1} & 0.82 & 1402 & 1419 & --- & --- & --- & 29.8 & --- \\
\texttt{gpt-4-1-mini} & --- & 1349 & 1231 & --- & --- & --- & --- & --- \\
\texttt{gpt-4-1-nano} & --- & 1306 & 1034 & --- & --- & --- & --- & --- \\
\texttt{gpt-4-turbo} & 0.64 & 1322 & --- & --- & 0.14 & --- & --- & 6.55 \\
\texttt{gpt-4o} & 0.75 & 1423 & 1484 & 8.18 & 0.12 & 3.27 & --- & 7.32 \\
\texttt{gpt-4o-mini} & 0.63 & 1295 & 950 & 6.72 & 0.11 & 3.11 & 24.9 & 6.25 \\
\texttt{gpt-5} & 0.87 & 1376 & 1301 & 8.60 & --- & --- & 47.6 & --- \\
\texttt{gpt-5-4} & 0.88 & 1429 & 2019 & --- & --- & --- & --- & --- \\
\texttt{gpt-5-4-mini} & --- & 1417 & 1726 & --- & --- & --- & --- & --- \\
\texttt{gpt-5-4-nano} & --- & 1336 & --- & --- & --- & --- & --- & --- \\
\texttt{gpt-5-mini} & --- & 1326 & 1301 & 8.31 & --- & --- & 35.6 & --- \\
\texttt{gpt-5-nano} & --- & 1250 & 866 & --- & --- & --- & --- & --- \\
\texttt{o3} & 0.85 & 1384 & --- & 8.39 & --- & --- & --- & --- \\
\texttt{o3-mini} & 0.79 & 1301 & --- & 6.15 & 0.17 & --- & --- & 6.51 \\
\texttt{o4-mini} & --- & 1338 & --- & 7.50 & --- & --- & --- & --- \\
\midrule
\multicolumn{9}{@{}l}{\raisebox{-0.25ex}{\includegraphics[height=1.8ex]{figures/neurips/logos/anthropic.pdf}}\ \textit{Anthropic}} \\
\texttt{claude-3-5-haiku} & 0.62 & 1303 & 1241 & 7.35 & 0.11 & 2.50 & --- & 6.79 \\
\texttt{claude-3-5-sonnet} & 0.78 & --- & --- & --- & --- & 2.36 & --- & 6.92 \\
\texttt{claude-3-7-sonnet} & 0.78 & --- & --- & --- & --- & --- & --- & 7.12 \\
\texttt{claude-3-7-sonnet\_thinking} & 0.78 & --- & --- & --- & --- & --- & --- & 7.22 \\
\texttt{claude-3-haiku} & 0.42 & 1214 & 848 & --- & 0.13 & 2.94 & --- & --- \\
\texttt{claude-3-opus} & 0.68 & --- & --- & --- & --- & 2.67 & --- & 6.36 \\
\texttt{claude-haiku-4-5} & --- & 1384 & --- & --- & --- & --- & --- & --- \\
\texttt{claude-opus-4-5} & 0.87 & 1462 & 1769 & --- & --- & --- & 44.4 & --- \\
\texttt{claude-opus-4-6} & 0.89 & 1467 & 1965 & --- & --- & --- & 52.1 & --- \\
\texttt{claude-sonnet-4} & 0.84 & 1384 & 1516 & 8.09 & --- & --- & --- & --- \\
\texttt{claude-sonnet-4-5} & 0.87 & 1450 & 1777 & --- & --- & --- & 39.9 & --- \\
\texttt{claude-sonnet-4-6} & 0.87 & 1444 & 1991 & --- & --- & --- & 44.2 & --- \\
\midrule
\multicolumn{9}{@{}l}{\raisebox{-0.25ex}{\includegraphics[height=1.8ex]{figures/neurips/logos/google.pdf}}\ \textit{Google}} \\
\texttt{gemini-1-5-pro} & 0.70 & --- & --- & --- & --- & 2.73 & --- & --- \\
\texttt{gemini-2-0-flash-001} & 0.78 & 1346 & 1252 & 7.15 & 0.15 & 3.17 & --- & 7.48 \\
\texttt{gemini-2-0-flash-lite-001} & 0.72 & 1345 & --- & --- & --- & 3.20 & --- & 6.60 \\
\texttt{gemini-2-0-pro} & 0.79 & --- & --- & --- & --- & 2.64 & --- & --- \\
\texttt{gemini-2-5-flash} & --- & 1399 & 1255 & 7.65 & --- & --- & 32.0 & --- \\
\texttt{gemini-2-5-pro} & 0.86 & 1448 & 1415 & 8.38 & --- & --- & 31.7 & --- \\
\texttt{gemma-2-27b-it} & 0.57 & 1291 & --- & --- & 0.16 & 3.77 & --- & 6.66 \\
\texttt{gemma-2-2b-it} & 0.16 & --- & --- & --- & --- & 4.63 & --- & --- \\
\texttt{gemma-2-9b-it} & 0.52 & 1257 & 920 & --- & 0.17 & 3.93 & --- & --- \\
\texttt{gemma-3-27b-it} & 0.68 & 1348 & 1256 & 7.99 & --- & 2.26 & --- & --- \\
\midrule
\multicolumn{9}{@{}l}{\raisebox{-0.25ex}{\includegraphics[height=1.8ex]{figures/neurips/logos/meta.pdf}}\ \textit{Meta}} \\
\texttt{llama-3-1-405b-instruct} & 0.73 & --- & --- & --- & --- & 3.39 & --- & 6.62 \\
\texttt{llama-3-1-70b-instruct} & 0.63 & 1257 & 851 & --- & 0.18 & 3.72 & --- & 6.62 \\
\texttt{llama-3-1-8b-instruct} & 0.44 & 1178 & 840 & --- & 0.19 & 3.76 & --- & --- \\
\texttt{llama-3-2-1b-instruct} & 0.12 & --- & --- & --- & --- & 2.81 & --- & --- \\
\texttt{llama-3-2-3b-instruct} & 0.22 & 1144 & --- & --- & 0.20 & 3.24 & --- & --- \\
\texttt{llama-3-3-70b-instruct} & 0.66 & 1286 & --- & --- & 0.13 & 2.87 & 37.1 & 6.06 \\
\texttt{llama-4-maverick} & 0.81 & 1307 & 944 & 6.20 & --- & 3.24 & 26.2 & --- \\
\texttt{llama-4-scout} & 0.74 & 1290 & 899 & --- & --- & 3.14 & --- & --- \\
\midrule
\multicolumn{9}{@{}l}{\raisebox{-0.25ex}{\includegraphics[height=1.8ex]{figures/neurips/logos/mistral.pdf}}\ \textit{Mistral}} \\
\texttt{mistral-7b-instruct-v0-1} & 0.26 & 1104 & --- & --- & 0.19 & 3.28 & --- & --- \\
\texttt{mistral-large-2407} & 0.66 & 1287 & 1400 & 6.90 & --- & --- & --- & --- \\
\texttt{mistral-large-2411} & 0.68 & 1276 & 1082 & 6.90 & 0.14 & --- & --- & 6.79 \\
\texttt{mistral-nemo} & 0.45 & 1090 & 966 & --- & 0.20 & 4.00 & --- & 7.24 \\
\texttt{mistral-small-24b-instruct-2501} & 0.66 & 1227 & --- & --- & 0.18 & 3.29 & --- & 6.41 \\
\midrule
\multicolumn{9}{@{}l}{\raisebox{-0.25ex}{\includegraphics[height=1.8ex]{figures/neurips/logos/qwen.pdf}}\ \textit{Qwen}} \\
\texttt{qwen-2-5-72b-instruct} & 0.72 & 1254 & --- & --- & 0.14 & 3.09 & 27.1 & 6.62 \\
\texttt{qwen-2-5-7b-instruct} & --- & --- & --- & --- & --- & --- & --- & 6.47 \\
\texttt{qwen-2-5-coder-32b-instruct} & --- & --- & --- & --- & --- & --- & --- & 6.56 \\
\texttt{qwen-max} & 0.76 & --- & --- & --- & --- & --- & --- & 6.59 \\
\texttt{qwen3-14b} & --- & --- & --- & --- & 0.13 & --- & --- & --- \\
\texttt{qwen3-235b-a22b} & 0.83 & 1324 & 1379 & 8.30 & --- & 3.89 & --- & --- \\
\texttt{qwen3-32b} & --- & 1306 & --- & --- & 0.15 & --- & --- & --- \\
\texttt{qwen3-8b} & --- & --- & --- & --- & 0.12 & --- & --- & --- \\
\texttt{qwq-32b} & 0.69 & 1296 & 1262 & 8.02 & --- & --- & --- & 7.06 \\
\midrule
\multicolumn{9}{@{}l}{\raisebox{-0.25ex}{\includegraphics[height=1.8ex]{figures/neurips/logos/deepseek.pdf}}\ \textit{DeepSeek}} \\
\texttt{deepseek-chat} & 0.76 & 1349 & --- & --- & 0.14 & 3.12 & 29.1 & 6.72 \\
\texttt{deepseek-chat-v3-0324} & 0.81 & 1390 & 1473 & 7.70 & 0.14 & 3.20 & 23.3 & --- \\
\texttt{deepseek-r1} & 0.84 & 1374 & 1500 & 8.30 & --- & --- & 34.5 & 7.18 \\
\texttt{deepseek-r1-distill-llama-70b} & --- & --- & --- & --- & --- & --- & --- & 6.71 \\
\texttt{deepseek-r1-distill-qwen-32b} & --- & --- & --- & --- & --- & --- & --- & 6.86 \\
\midrule
\multicolumn{9}{@{}l}{\raisebox{-0.25ex}{\includegraphics[height=1.8ex]{figures/neurips/logos/cohere.pdf}}\ \textit{Cohere}} \\
\texttt{command-a} & --- & 1337 & 1184 & --- & --- & --- & --- & --- \\
\texttt{command-r-08-2024} & --- & 1209 & --- & --- & --- & 2.98 & --- & --- \\
\texttt{command-r-plus-08-2024} & --- & 1263 & --- & --- & 0.23 & 3.08 & --- & --- \\
\texttt{command-r7b-12-2024} & --- & --- & --- & --- & --- & 3.35 & --- & --- \\
\midrule
\multicolumn{9}{@{}l}{\raisebox{-0.25ex}{\includegraphics[height=1.8ex]{figures/neurips/logos/nvidia.pdf}}\ \textit{NVIDIA}} \\
\texttt{llama-3-1-nemotron-70b-instruct} & 0.63 & 1277 & --- & --- & --- & --- & --- & 7.25 \\
\midrule
\multicolumn{9}{@{}l}{\raisebox{-0.25ex}{\includegraphics[height=1.8ex]{figures/neurips/logos/microsoft.pdf}}\ \textit{Microsoft}} \\
\texttt{phi-4} & 0.70 & 1210 & --- & 6.26 & 0.15 & 3.59 & 18.5 & 6.72 \\
\midrule
\multicolumn{9}{@{}l}{\raisebox{-0.25ex}{\includegraphics[height=1.8ex]{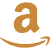}}\ \textit{Amazon}} \\
\texttt{nova-lite-v1} & --- & --- & --- & --- & --- & --- & --- & 6.05 \\
\texttt{nova-pro-v1} & --- & --- & --- & --- & --- & --- & --- & 6.58 \\
\end{longtable}
\normalsize

\clearpage
\section{Greedy Algorithm for the DAT} \label{app:greedy}
The greedy algorithm in \Cref{fig:greedy-alg} starts with a random word from a valid vocabulary set $V$ (e.g., the $\sim\!42{,}000$ single-token English nouns in WordNet $\cap$ GloVe 840B), then proceeds to \emph{minimize the mean cosine similarity} (equivalently, \emph{maximize the mean cosine dissimilarity}, matching Equation~\eqref{eq:dat}) of all subsequent words, while avoiding any repetition.

\Cref{fig:greedy-baseline} reports the result of running the algorithm for 120 independent seeds, and scoring each sequence according to Equation~\eqref{eq:dat} under GloVe, FastText, and Sentence-BERT. The three scores are averaged per trial and across embeddings to produce the values in Figure~\ref{fig:greedy-baseline} (mean $= 94.1$, std.\ dev.\ $= 0.9$). The LLM pool, averaged the same way over 2{,}082 trials from 55 models, has mean $= 81.1$ (std.\ dev.\ $= 3.0$), and humans mean $= 78.4$ (std.\ dev.\ $= 6.4$); human scores are on the GloVe scale only, because \cite{Olson2021NamingCreativity} reports only final scores---not raw word lists. As shown, the algorithm substantially exceeds the mean LLM and human scores.

\clearpage
\section{Evidence for Jagged Creativity} \label{app:dat-repetition}
This section gives evidence for jagged creativity (\Cref{sec:results}) from raw DAT responses. \Cref{tab:dat-repetition} shows example responses from three models at each sampling temperature: \texttt{gpt-5-4} and \texttt{claude-sonnet-4-6}, two of the highest-scoring models on the DAT, repeat the same words across nearly all of their 40 responses, even at $T = 2.0$, while \texttt{command-r-plus-08-2024} rarely repeats a word. \Cref{tab:base-vs-instruct} compares five base models, given the DAT as a text completion, with their instruction-tuned versions: instruction tuning raises the mean DAT score from 77.4 to 81.4, but cuts the number of distinct words a model uses across 40 responses by about two-thirds (from 296 to 97 on average).
\begin{table}[h]
\centering
\caption{\textbf{Repeated words in DAT responses across temperatures.} Two responses per model and temperature. Bold words appear in at least half of the model's 40 responses at that temperature, and Distinct counts unique words across all 40 (out of 400).}
\label{tab:dat-repetition}
\scriptsize
\setlength{\tabcolsep}{3pt}
\begin{tabular}{@{}lcp{1.55cm}cp{7.15cm}@{}}
\toprule
Model & $T$ & Top word & Distinct & Example runs \\
\midrule
\texttt{gpt-5-4} & 1.0 & quasar (39/40) & 98 & \textbf{quasar}, \textbf{moss}, anvil, melody, verdict, bacterium, glacier, silk, hunger, \textbf{algorithm} \\
DAT rank 1 of 55 &  & &  & \textbf{quasar}, thimble, maggot, \textbf{justice}, tornado, sapphire, fungus, \textbf{algorithm}, canyon, lullaby \\
\cmidrule(l){2-5}
 & 1.5 & quasar (38/40) & 104 & \textbf{quasar}, teaspoon, democracy, fungus, ankle, silence, volcano, \textbf{algorithm}, perfume, cradle \\
 &  & &  & \textbf{quasar}, spoon, \textbf{moss}, justice, volcano, \textbf{algorithm}, perfume, skeleton, thunder, market \\
\cmidrule(l){2-5}
 & 2.0 & quasar (40/40) & 99 & \textbf{quasar}, thimble, \textbf{moss}, equation, volcano, violin, citizen, teaspoon, nightmare, antibody \\
 &  & &  & \textbf{quasar}, \textbf{moss}, anvil, harmony, bacterium, labyrinth, teardrop, currency, skeleton, thunder \\
\midrule
\texttt{claude-sonnet-4-6} & 1.0 & justice (40/40) & 131 & \textbf{justice}, membrane, avalanche, spoon, frequency, \textbf{debt}, fossil, whisper, throne, \textbf{fungus} \\
DAT rank 5 of 55 &  & &  & \textbf{justice}, kelp, turbine, melancholy, socket, chromosome, fable, gravel, oboe, \textbf{debt} \\
\cmidrule(l){2-5}
 & 1.5 & justice (40/40) & 117 & \textbf{justice}, membrane, glacier, tambourine, \textbf{debt}, spore, labyrinth, oxygen, anvil, grief \\
 &  & &  & \textbf{justice}, \textbf{fungus}, telescope, grief, anvil, chromosome, tide, labyrinth, ember, \textbf{debt} \\
\cmidrule(l){2-5}
 & 2.0 & justice (40/40) & 124 & \textbf{justice}, \textbf{fungus}, lever, sorrow, pixel, anvil, treaty, cobweb, neuron, glacier \\
 &  & &  & \textbf{justice}, \textbf{fungus}, locomotive, \textbf{grief}, pixel, hammock, \textbf{debt}, chromosome, avalanche, flute \\
\midrule
\texttt{command-r-plus-08-2024} & 1.0 & galaxy (7/40) & 351 & apple, soul, revolution, rainbow, politics, geology, purveyor, inequality, epiphany, pandemic \\
DAT rank 14 of 55 &  & &  & entropy, peace, hologram, echo, awareness, petrichor, capybara, insurrection, titanium, constellation \\
\cmidrule(l){2-5}
 & 1.5 & galaxy (7/40) & 359 & carrot, pencil, painting, friendship, sandwich, scientist, curtain, meaning, song, threshold \\
 &  & &  & truth, cherry, imperative, neutron, shone, alchemy, oblong, grandiose, bubbly, pajamas \\
\cmidrule(l){2-5}
 & 2.0 & galaxy (6/40) & 359 & abundance, charity, criminal, equinox, hatred, innocence, knowledge, luggage, molecule, ocean \\
 &  & &  & cephalopod, hatred, whistle, jubilee, avocado, pyramid, nimbus, saunter, bewilder, annihilate \\
\bottomrule
\end{tabular}

\end{table}

\begin{table}[h]
\centering
\caption{\textbf{Base vs.\ instruction-tuned models on the DAT ($T = 1.0$, top-$p = 1$, top-$k = 0$).} Base models receive the same DAT prompt framed as a text completion. Distinct counts unique words across 40 responses (out of 400), and the most repeated word is shown with the number of responses containing it.}
\label{tab:base-vs-instruct}
\small
\setlength{\tabcolsep}{5pt}
\begin{tabular}{@{}lrrrrll@{}}
\toprule
 & \multicolumn{2}{c}{DAT score} & \multicolumn{2}{c}{Distinct words} & \multicolumn{2}{c}{Most repeated word} \\
\cmidrule(lr){2-3} \cmidrule(lr){4-5} \cmidrule(l){6-7}
Model & Base & Instruct & Base & Instruct & Base & Instruct \\
\midrule
Llama 3.2 3B & 82.0 & 84.0 & 341 & 13 & food (3/40) & basket (40/40) \\
Llama 3.1 8B & 83.0 & 87.6 & 331 & 236 & car (3/40) & fjord (19/40) \\
Qwen2.5 7B & 79.4 & 76.1 & 298 & 120 & book (7/40) & ocean (33/40) \\
Mistral 7B v0.1 & 77.2 & 81.2 & 351 & 15 & dog (5/40) & apple (40/40) \\
Gemma 2 9B & 65.5 & 77.9 & 159 & 99 & cat (11/40) & silence (35/40) \\
\midrule
Mean $\pm$ SEM & 77.4$\pm$3.1 & 81.4$\pm$2.1 & 296$\pm$35 & 97$\pm$41 & (5.8$\pm$1.5/40) & (33.4$\pm$3.9/40) \\
\bottomrule
\end{tabular}

\end{table}

\clearpage
\section{Prompts} \label{app:prompts}
Below, we report the exact prompts used to administer the DAT, CDAT, PACE, and RAT. On CDAT, the cue word varies across 50 cues drawn from semantically diverse categories (Physical world, Animals, Food \& drink, Body, Emotions, \dots); for PACE the seed word varies across 50 seeds from the same category set. We show one instantiation with cue/seed \texttt{"rock"} for illustration.

\subsection{DAT}
\begin{plainpromptbox}
Please enter 10 words that are as different from each other as possible, in all meanings and uses of the words. Only use single nouns. Do not use proper nouns (names, places, brands). Do not use variations of the same word (e.g., don't use both `run' and `running').\\[4pt]
Respond with ONLY a JSON array of exactly 10 words, like: ["word1", "word2", "word3", "word4", "word5", "word6", "word7", "word8", "word9", "word10"]
\end{plainpromptbox}

\subsection{CDAT (cue: rock)}
\begin{plainpromptbox}
Please enter 10 words that are as different from each other as possible, in all meanings and uses of the words, yet semantically associated with the following cue word: "rock". Only use single nouns. Do not use proper nouns. Do not use the cue word itself or variations of it. Respond with ONLY a JSON array of exactly 10 words, like: ["word1", "word2", "word3", "word4", "word5", "word6", "word7", "word8", "word9", "word10"]
\end{plainpromptbox}

\subsection{PACE Stage 1 (seed: rock)}
\begin{plainpromptbox}
Starting with the word "rock", generate three different words that directly associate with this initial word only (not with each other). Please put down only single words, and do not use proper nouns (such as names, brands, etc.). For each word, provide a brief explanation of its connection to "rock". Return in JSON format:\\[4pt]
\{"results": [\{"word": "", "reason": ""\}, \{"word": "", "reason": ""\}, \{"word": "", "reason": ""\}]\}
\end{plainpromptbox}

\subsection{PACE Stage 2 (seed: rock, first-association: stone)}
\begin{plainpromptbox}
Starting with the word pair "rock" -> "stone", generate a chain of 20 words where each new word should be associated with ONLY the word immediately before it. Generate the third word based on "stone", then generate the fourth word based on your third word, and so on. Please put down only single words, and do not use proper nouns (such as names, brands, etc.). For each word, provide a brief explanation of its connection to the previous word. Return in JSON format with exactly 20 entries:\\[4pt]
\{"results": [\{"word": "stone", "reason": "<stage-1 reason>"\}, \{"word": "", "reason": ""\}, \dots\ ]\}
\end{plainpromptbox}

\subsection{RAT (item: cream, skate, water)}
\begin{plainpromptbox}
What single word can be combined with each of "cream", "skate", and "water" to form a compound word or common phrase?\\[4pt]
Respond with ONLY the single answer word in lowercase. No explanation.
\end{plainpromptbox}

\clearpage
\section{Disclosure of Large Language Model Usage} \label{app:llm-usage}
In this paper, large language models (LLMs) were used to assist in the code implementation, plotting and figure generation, and the reporting (but not analysis) of experiment results.

\section{Codebase Release} \label{app:code}
The full codebase --- including the evaluation pipeline (DAT, CDAT, PACE, RAT), embedding scoring, correlation and partial-correlation analysis, the figure-generation scripts, and the per-model raw responses --- is available at:
\begin{center}
\url{https://github.com/samjschapiro/do-semantic-distance-tests-actually-predict-creativity-neurips-26}
\end{center}

\section{Compute Resources} \label{app:compute}
All experiments were run as inference-time API calls through OpenRouter; no model training or fine-tuning was performed. Total spend across the full eval campaign (DAT, CDAT, PACE, and RAT on 55 instruction-tuned LLMs at multiple temperatures, including iteration during pipeline development) was approximately \$150 in API charges. Embedding-based scoring (GloVe, FastText, Sentence-BERT) and all downstream correlation and figure generation ran on a single laptop CPU; per-pipeline wall-clock is dominated by API latency.

\section{Broader Impacts} \label{app:broader-impacts}
This paper presents work whose goal is to advance the field of machine learning. There are many potential societal consequences of our work, none of which we feel must be specifically highlighted here.

\newpage
\section*{NeurIPS Paper Checklist}

\begin{enumerate}

\item {\bf Claims}
    \item[] Question: Do the main claims made in the abstract and introduction accurately reflect the paper's contributions and scope?
    \item[] Answer: \answerYes{}
    \item[] Justification: The findings stated in the abstract and \Cref{sec:intro} (no test predicts output diversity, suggesting jagged creativity; the RAT is a capability proxy; the DAT best predicts creative writing but no single test predicts all constructs) match the results in \Cref{sec:results}, and the theoretical bound is proved in \Cref{thm:spec-ceiling} / \Cref{app:spec-bound}.

\item {\bf Limitations}
    \item[] Question: Does the paper discuss the limitations of the work performed by the authors?
    \item[] Answer: \answerYes{}
    \item[] Justification: \Cref{sec:limitations} discusses the small $n$ of several benchmarks, the restriction of the model pool to instruction-tuned LLMs (so findings may not transfer to base models), and the limits of minimal tests relative to open-world evaluations.

\item {\bf Theory assumptions and proofs}
    \item[] Question: For each theoretical result, does the paper provide the full set of assumptions and a complete (and correct) proof?
    \item[] Answer: \answerYes{}
    \item[] Justification: \Cref{thm:spec-ceiling} states all assumptions explicitly; the proof is given in \Cref{app:spec-bound}.

\item {\bf Experimental result reproducibility}
    \item[] Question: Does the paper fully disclose all the information needed to reproduce the main experimental results of the paper to the extent that it affects the main claims and/or conclusions of the paper (regardless of whether the code and data are provided or not)?
    \item[] Answer: \answerYes{}
    \item[] Justification: \Cref{sec:method} specifies the model pool, sampling parameters, embedding models, and partial-correlation computation. Exact prompts for the DAT, CDAT, PACE, and RAT are given in \Cref{app:prompts}. Per-model raw test scores and benchmark scores are listed in \Cref{tab:per-model-scores} and \Cref{tab:per-model-benchmarks}.

\item {\bf Open access to data and code}
    \item[] Question: Does the paper provide open access to the data and code, with sufficient instructions to faithfully reproduce the main experimental results, as described in supplemental material?
    \item[] Answer: \answerYes{}
    \item[] Justification: The full codebase --- evaluation pipeline, scoring, correlation analysis, and figure generation --- is released at the URL given in \Cref{app:code}.

\item {\bf Experimental setting/details}
    \item[] Question: Does the paper specify all the training and test details (e.g., data splits, hyperparameters, how they were chosen, type of optimizer, etc.) necessary to understand the results?
    \item[] Answer: \answerYes{}
    \item[] Justification: \Cref{sec:method} gives temperatures, sampling parameters, the number of trials per model, the appropriateness gate threshold and FDR correction for CDAT, and the capability control (MMLU-Pro) used for specificity.

\item {\bf Experiment statistical significance}
    \item[] Question: Does the paper report error bars suitably and correctly defined or other appropriate information about the statistical significance of the experiments?
    \item[] Answer: \answerYes{}
    \item[] Justification: We report two-sided $p$-values for all validity and specificity correlations using the standard Pearson $t$-test with $\mathrm{df} = n - 2 - k$ (\Cref{sec:method}), with significance levels marked in \Cref{tab:correlations}. Per-model scores in \Cref{tab:per-model-scores} are reported as mean~$\pm$~SEM.

\item {\bf Experiments compute resources}
    \item[] Question: For each experiment, does the paper provide sufficient information on the computer resources (type of compute workers, memory, time of execution) needed to reproduce the experiments?
    \item[] Answer: \answerYes{}
    \item[] Justification: \Cref{app:compute} reports total OpenRouter API spend (approximately \$150) and notes that downstream scoring and figure generation run on a single laptop CPU; no training or fine-tuning was performed.

\item {\bf Code of ethics}
    \item[] Question: Does the research conducted in the paper conform, in every respect, with the NeurIPS Code of Ethics \url{https://neurips.cc/public/EthicsGuidelines}?
    \item[] Answer: \answerYes{}
    \item[] Justification: The work uses only public benchmark scores and inference-time API calls to commercially available LLMs; no human subjects data was collected, and no private model training was performed.

\item {\bf Broader impacts}
    \item[] Question: Does the paper discuss both potential positive societal impacts and negative societal impacts of the work performed?
    \item[] Answer: \answerYes{}
    \item[] Justification: \Cref{app:broader-impacts} provides a brief broader-impacts statement.

\item {\bf Safeguards}
    \item[] Question: Does the paper describe safeguards that have been put in place for responsible release of data or models that have a high risk for misuse (e.g., pretrained language models, image generators, or scraped datasets)?
    \item[] Answer: \answerNA{}
    \item[] Justification: The paper does not release any new model or scraped dataset.

\item {\bf Licenses for existing assets}
    \item[] Question: Are the creators or original owners of assets (e.g., code, data, models), used in the paper, properly credited and are the license and terms of use explicitly mentioned and properly respected?
    \item[] Answer: \answerYes{}
    \item[] Justification: All benchmarks (Arena CW, EQ-Bench CW, Mazur CW, Hivemind, NoveltyBench Utility, Kombine, LiveIdeaBench, MMLU-Pro), the RAT items, and embedding resources (GloVe, FastText, Sentence-BERT) are cited in \Cref{sec:background} and \Cref{sec:method}. Source URLs and snapshot identifiers for the leaderboards used (Mazur GitHub commit, TIGER-Lab MMLU-Pro CSV, EQ-Bench JS payload) are documented in \Cref{app:per-model-benchmarks}.

\item {\bf New assets}
    \item[] Question: Are new assets introduced in the paper well documented and is the documentation provided alongside the assets?
    \item[] Answer: \answerYes{}
    \item[] Justification: We release the evaluation code together with every model's raw responses and scores, documented in a README that covers setup, the data layout, and how to reproduce each table and figure (\Cref{app:code}).

\item {\bf Crowdsourcing and research with human subjects}
    \item[] Question: For crowdsourcing experiments and research with human subjects, does the paper include the full text of instructions given to participants and screenshots, if applicable, as well as details about compensation (if any)?
    \item[] Answer: \answerNA{}
    \item[] Justification: The work does not involve crowdsourcing or human subjects; the human DAT distribution referenced in \Cref{fig:greedy-baseline} is taken from a published source (\cite{Olson2021NamingCreativity}).

\item {\bf Institutional review board (IRB) approvals or equivalent for research with human subjects}
    \item[] Question: Does the paper describe potential risks incurred by study participants, whether such risks were disclosed to the subjects, and whether Institutional Review Board (IRB) approvals (or an equivalent approval/review based on the requirements of your country or institution) were obtained?
    \item[] Answer: \answerNA{}
    \item[] Justification: No human subjects were recruited for this work.

\item {\bf Declaration of LLM usage}
    \item[] Question: Does the paper describe the usage of LLMs if it is an important, original, or non-standard component of the core methods in this research? Note that if the LLM is used only for writing, editing, or formatting purposes and does not impact the core methodology, scientific rigorousness, or originality of the research, declaration is not required.
    \item[] Answer: \answerYes{}
    \item[] Justification: LLMs are the objects of study; the 55 models and their API settings are described in \Cref{sec:method}. \Cref{app:llm-usage} discloses how LLMs assisted this work (code implementation, figure generation, and reporting, but not analysis).

\end{enumerate}

\end{document}